\documentclass[12pt,final]{colt2019} 

\usepackage{times}
\usepackage[utf8]{inputenc} 
\usepackage[T1]{fontenc}    
\usepackage{url}            
\usepackage{booktabs}       
\usepackage{amsfonts}       
\usepackage{nicefrac}       
\usepackage{microtype}      
\usepackage{mkolar_definitions}
\usepackage{xspace}
\usepackage{amssymb}
\usepackage{amsmath}
\usepackage{algorithm}
\usepackage{algorithmic}
\usepackage{color}
\usepackage{enumitem}
\usepackage{comment}
\usepackage{bm}
\usepackage[scaled=.9]{helvet}
\usepackage{url}

\newcommand{\defeq}{\triangleq}

\newtheorem*{theorem*}{Theorem}
\newtheorem*{runex*}{Running Example}
\newtheorem*{ex*}{Example}
\newtheorem*{remark*}{Remark}

\usepackage{natbib}
\bibpunct{(}{)}{;}{a}{,}{,}

\newcount\Comments  
\definecolor{darkgreen}{rgb}{0,0.5,0}
\definecolor{darkred}{rgb}{0.7,0,0}
\definecolor{teal}{rgb}{0.3,0.8,0.8}
\definecolor{orange}{rgb}{1.0,0.5,0.0}
\definecolor{purple}{rgb}{0.8,0.0,0.8}
\newcommand{\kibitz}[2]{\ifnum\Comments=1{\textcolor{#1}{\textsf{\footnotesize #2}}}\fi}
\newcommand{\wen}[1]{\kibitz{darkred}{[WS: #1]}}
\newcommand{\nan}[1]{\kibitz{teal}{[NJ: #1]}}
\newcommand{\akshay}[1]{\kibitz{darkgreen}{[AK: #1]}}
\newcommand{\alekh}[1]{\kibitz{purple}{[AA: #1]}}

\newcommand{\olive}{\textsc{Olive}\xspace}
\newcommand{\variation}[1]{\nbr{#1}_{\textrm{TV}}}
\newcommand{\OP}{\texttt{OP}}

\newcommand{\tMcal}{\widetilde{\Mcal}}
\newcommand{\tFcal}{\widetilde{\Fcal}}
\newcommand{\tQcal}{\widetilde{\Qcal}}

\newcommand{\bmp}{\textbf{p}}
\newcommand{\bmx}{\textbf{x}}
\newcommand{\bma}{\textbf{a}}
\newcommand{\parent}{\textup{pa}}

\newcommand{\sd}{\eta} 
\DeclareMathOperator{\E}{\mathbb{E}}

\title[Model-based RL in CDPs]{Model-based RL in Contextual Decision Processes: PAC bounds and Exponential Improvements over Model-free Approaches}

\coltauthor{\Name{Wen Sun} \Email{wensun@cs.cmu.edu} \\
\addr Carnegie Mellon University, Pittsburgh, PA  
\AND
\Name{Nan Jiang} \Email{nanjiang@illinois.edu} \\
\addr University of Illinois at Urbana-Champaign, Urbana, IL 
\AND
\Name{Akshay Krishnamurthy} \Email{akshay@cs.umass.edu}\\
\addr Microsoft Research, New York, NY 
\AND
\Name{Alekh Agarwal} \Email{alekha@microsoft.com}\\
\addr Microsoft Research, Redmond, WA 
\AND
\Name{John Langford} \Email{jcl@microsoft.com} \\
\addr Microsoft Research, New York, NY
}

\begin{document}

\maketitle

\begin{abstract}%
  We study the sample complexity of model-based reinforcement learning (henceforth RL)
in general contextual decision processes that require strategic exploration to find a near-optimal policy.  We design new algorithms
for RL with a generic model class and analyze their statistical
properties.  Our algorithms 
have sample complexity governed by a new
structural parameter called the \emph{witness rank}, which we show to
be small in several settings of interest, including factored MDPs. We also show that the witness rank is
never larger than the recently proposed \emph{Bellman rank} parameter
governing the sample complexity of the model-free
algorithm \olive~\citep{jiang2017contextual}, the only other provably
sample-efficient algorithm for global exploration at this level of generality. Focusing on
the special case of factored MDPs, we prove an exponential lower bound
for a general class of model-free approaches, including \olive, which, when combined
with our algorithmic results, demonstrates exponential separation
between model-based and model-free RL in some rich-observation
settings.

\end{abstract}

\setlist[enumerate]{topsep=0pt,itemsep=-1ex,partopsep=1ex,parsep=1ex,leftmargin=*}
\setlist[itemize]{topsep=0pt,itemsep=-1ex,partopsep=1ex,parsep=1ex,leftmargin=*}
\setlength{\abovedisplayskip}{2pt}
\setlength{\belowdisplayskip}{2pt}

\section{Introduction}
\label{sec:introduction}

Reinforcement learning algorithms can be broadly categorized as
model-based or model-free methods. Methods in the former family
explicitly model the environment dynamics and then use planning
techniques to find a near-optimal policy. In contrast, the latter
family models much less, typically only an optimal policy and its value. Algorithms from both
families have seen substantial empirical success, but we lack a
rigorous understanding of the tradeoffs between them, making algorithm
selection difficult for practitioners. This paper provides a new
understanding of these tradeoffs, via a comparative analysis between
model-based and model-free methods in general RL settings.

Conventional wisdom and intuition suggests that model-based methods
are more sample-efficient than model-free methods, since they leverage
more supervision. This argument is supported by classical
control-theoretic settings like the linear quadratic regulator, where
state-of-the-art model-based methods have better dimension dependence
than contemporary model-free ones~\citep{tu2018gap}.
On the other hand, since models typically have more degrees
of freedom (e.g., parameters) and can waste effort on unimportant
elements of the environment, one might suspect that model-free methods
have better statistical properties. Indeed, recent work in tabular
Markov Decision Processes (MDPs) suggest that there is almost no
sample-efficiency gap between the two families~\citep{jin2018q}. Even
worse, 
in complex environments where function
approximation and global exploration are essential, the only algorithms with sample complexity
guarantees are model-free~\citep{jiang2017contextual}. In such
environments, which of these competing perspectives applies?

To answer this question, we study model-based RL
in episodic \emph{contextual decision processes} (CDPs) where
high-dimensional observations are used for decision
making and the learner needs to perform strategic exploration to find a near-optimal policy. 
For model-based algorithms, we assume access to a class $\Mcal$ of
models and that the true environment is representable by the class,
while for model-free algorithms, we assume access to a class of value
functions that realizes the optimal value function (with analogous assumptions for policy-based methods). 
Under such assumptions, we posit:

\vspace{-0.2cm}
\begin{quote}
\emph{Model-based methods rely on stronger function-approximation
  capabilities but can be exponentially more sample efficient than their
  model-free counterparts.}
\end{quote}
\vspace{-0.2cm}

\vspace{0.05cm}
Our contributions provide evidence for this thesis and can be summarized as follows:
\begin{enumerate}[itemsep=0pt, leftmargin=*]
 \item We show that there exist MDPs where (1) all model-free methods, given a value function class satisfying the above 
 realizability condition,
 incur exponential sample complexity (in horizon); and (2) there exist model-based methods that, given a model class containing the true model, obtain polynomial sample complexity. In fact, these MDPs belong to the well-studied
   \emph{factored MDPs}~\citep{kearns1999efficient}, which we
   use as a running example throughout the paper.
\item We design a new model-based algorithm for general CDPs and show that it has sample
  complexity governed by a new structural parameter, the \emph{witness
    rank}. We further show that many concrete settings, including tabular and low rank MDPs, reactive POMDPs, and
  reactive PSRs have a small witness rank. This algorithm is the first
  provably-efficient model-based algorithm that does not rely on
  tabular representations or highly structured control-theoretic
  settings.
\item We compare our algorithm and the witness rank with the
  model-free algorithm \olive~\citep{jiang2017contextual} and the
  Bellman rank, the only other algorithm and structural parameter
  at this level of generality. We show that the witness rank is never
  larger, and can be exponentially smaller than the Bellman rank. In particular, our algorithm has polynomial sample complexity in factored MDPs, an exponential gain over \olive and any other realizability-based model-free algorithm.

\end{enumerate}
\vspace{0.1cm}

The caveat in our thesis is that model-based methods rely on strong
realizability assumptions. In the rich environments we study, where
function approximation is essential, some form of realizability is
necessary (see Proposition 1 in~\citet{krishnamurthy2016pac}), but our
model-based assumption (See~\pref{ass:realizable}) is strictly
stronger than prior value-based
analogs~\citep{antos2008learning,krishnamurthy2016pac}.
On the other hand, our results precisely quantify the tradeoffs
between model-based and model-free approaches, which may guide
future empirical efforts.



\section{Preliminaries}
\label{sec:preliminaries}
We study \emph{Contextual Decision Processes} (CDPs), a general
sequential decision making setting where the agent optimizes long-term
reward by learning a policy that maps from rich observations
(e.g., raw-pixel images) to actions. The term CDP was
proposed by \cite{krishnamurthy2016pac} and extended by
\cite{jiang2017contextual}, 
with CDPs capturing broad classes of RL problems allowing rich observation spaces
including (Partially Observable) MDPs and Predictive State
Representations.  Please
see the above references for further background.

\paragraph{Notation.}
We use $[N] \defeq \{1,\ldots,N\}$. For a finite set
$S$, $\Delta(S)$ is the set of distributions over $S$, and $U(S)$ is the uniform distribution over $S$. For a function $f:S\to\RR$, $\|f\|_{\infty}$ denotes $\sup_{s\in S}\abr{f(s)}$. 

\subsection{Basic Definitions} \label{sec:basic}
Let $H\in \NN$ denote a time horizon and let $\Xcal$ be a
large context space of unbounded size, partitioned into subsets
$\Xcal_1,\ldots,\Xcal_{H+1}$.  A finite horizon episodic CDP is a tuple
$(\Xcal, \Acal, R, P)$ consisting of a (partitioned) context space
$\Xcal$, an action space $\Acal$, a transition operator $P: \{\bot\}
\cup (\Xcal \times\Acal) \to \Delta(\Xcal)$, and a reward function $R:
\Xcal \times \Acal \to \Delta(\Rcal)$ with $\Rcal\subseteq [0,1]$.\footnote{We assume Markov transitions w.l.o.g., since context may encode history.}
We assume a \emph{layered Markovian}
structure, so that for any $h \in [H]$, $x_h \in \Xcal_h$ and $a \in
\Acal$, the future context and the reward distributions are characterized by $x_h,a$ and moreover $P_{x_h,a} \defeq P(x_h,a) \in \Delta(\Xcal_{h+1})$.
We use $P_0 \defeq P(\bot) \in \Delta(\Xcal_1)$ to denote the initial
context distribution, and we assume $|\Acal|=K$
throughout.
\footnote{Partitioning the context space by time
  allows us to capture more general time-dependent dynamics, reward,
  and policy.} Note that the layering of contexts allows us to implicitly model the level $h$ as part of the context.


A policy $\pi:\Xcal \to \Delta(\Acal)$ maps each context to a
distribution over actions. By executing this policy in the CDP for
$h-1$ steps, we naturally induce a distribution over $\Xcal_h$, and we
use $\EE_{x_h \sim \pi}[\cdot]$ to denote the expectation with respect
to this distribution.  A policy ${\pi}$ has associated value
and action-value functions $V^{\pi}:\Xcal \to \RR^+$ and
$Q^\pi:\Xcal\times\Acal \to \RR^+$, defined as
\begin{align*}
\forall h \in [H], x \in \Xcal_h:~ V^{\pi}(x) \defeq \EE_{\pi}\sbr{\sum_{t = h}^H  r_t  \,\Big|\, x_h = x}, ~ Q^{\pi}(x, a) \defeq \EE_{r\sim R(x,a)}[r] + \EE_{x'\sim P_{x,a}}\left[V^{\pi}(x')\right],
\end{align*}
Here, the expectation is over randomness in the environment and the
policy, with actions sampled by $\pi$.  Note that there is no need to index $V$
and $Q$ by the level $h$ since it is encoded in the context.
The value of a policy $\pi$ is $v^\pi \defeq
\EE_{x_1\sim P_0}\sbr{V^\pi(x_1)}$ and the goal is to find a policy
$\pi$ that maximizes $v^\pi$.

For regularity, we assume that almost surely
$\sum_{h=1}^H r_h \le 1$.





\begin{runex*}
As a running example, we consider \emph{factored
  MDPs}~\citep{kearns1999efficient}. Let $d \in \NN$ and let $\Ocal$
be a small finite set. Define the context space $\Xcal \defeq
[H]\times\Ocal^d$, with the natural partition by time. For a state $x \in
\Xcal$ we use $x[i]$ for $i \in [d]$ to denote the value of $x$ on the  $i^{\textrm{th}}$
state variable (ignoring the time step $h$), and similar notation for a
subset of state variables.  For each state variable $i \in [d]$, the
parents of $i$, $\parent_i \subseteq [d]$ are the subset of state
variables that directly influence $i$. In factored MDPs, the
transition dynamics $P$ factorize according to the parent
relationships:
\begin{align}\label{eq:factorization}
\forall h, x \in \Xcal_h, x'\in\Xcal_{h+1}, a\in \Acal, \ \ P(x' \mid x,a) = \prod_{i=1}^d P^{(i)}[x'[i] \mid x[\parent_i] ,a, h]
\end{align}
for conditional distributions $\{P^{(i)}\}_{i=1}^d$ of the appropriate
dimensions. Note that we always condition on the time point $h$ to
allow for non-stationary transitions. This transition
operator has  $L \defeq \sum_{i=1}^d HK\cdot
|\Ocal|^{1+|\parent_i|}$ parameters, which can be much smaller than
$dHK|\Ocal|^{1+d}$ for an unfactorized process on $|\Ocal|^d$
states.\footnote{Actually the full unfactored process has
  $HK|\Ocal|^{2d}$ parameters. Here we are assuming that the state
  variables are conditionally independent given the previous state and
  action.} When $|\parent_i|$ is small for all $i$, we can
expect algorithms with low sample
complexity. Indeed~\citet{kearns1999efficient} show that factored MDPs
can be PAC learned with $\textrm{poly}(H,K,L,\epsilon,\log(1/\delta))$
samples in the average and discounted reward settings. For more recent development in this line of research, we refer the readers to  \citet{diuk2009adaptive, nguyen2013online, osband2014near, guo2017sample} and the references therein.
\end{runex*}

\subsection{Model Class}
Since we are interested in general CDPs with large state spaces (i.e., non-tabular setting), we
equip model-based algorithms with a model class $\Mcal$, where all models in $\Mcal$ share the same $\Xcal$ and $\Acal$ but can differ in reward function $R$ and transition operator $P$. 
The environment reward and dynamics are called the \emph{true model}
and denoted $M^\star \defeq (R^\star, P^\star)$.  For a model $M \in
\Mcal$, $\pi_M,V_M,Q_M$, and $v_M$ are the optimal policy, value
function, action-value function, and value \emph{in the model $M$},
respectively. These objects are purely functions of $M$ and do not depend on the environment. For the true model $M^\star$,
these quantities are denoted $\pi^\star, V^\star, Q^\star, v^\star$,
suppressing subscripts.  For $M\defeq(R, P)$, we denote $(r,x')\sim
M_{x,a}$ as sampling a reward and next context from $M$: $r\sim
R(x,a), x'\sim P_{x,a}$.  We use $x_h\sim \pi$ to denote a state sampled by executing $\pi$ in the true environment $M^\star$ for $h-1$
steps.

We use $\texttt{OP}$ (for Optimal Planning) to represent the operator
that maps a model $M$ to its optimal $Q$ function and its optimal policy, that is
$\texttt{OP}(M) \defeq ({Q}_M, \pi_M)$. We denote $\texttt{OP}(\Mcal) \defeq \left\{{Q},\pi: \exists  M\in\Mcal\ s.t.\ \texttt{OP}(M) = (Q,\pi) \right\}$ as the
set of optimal $Q$ functions and optimal policies derived from the class $\Mcal$.\footnote{As $\pi_M$ is determined by $Q_M$, i.e., $\pi_M(x) = \arg\max_a Q_M(x,a)$, we sometimes overload notation and use $\Qcal = \OP(\Mcal)$ to represent the set of optimal Q functions derived from $\Mcal$. }
Throughout the paper, when we compare model-based and
model-free methods, we use $\Mcal$ as input for
the former and $\texttt{OP}(\Mcal)$ for the latter.




We assume the model class has \emph{finite (but exponentially large) cardinality and is realizable}. 
\begin{assum}[Realizability of $\Mcal$]
\label{ass:realizable}
We assume the model class $\Mcal$ contains the true model $M^\star$.
\end{assum}





The finiteness assumption is made only to simplify presentation and can be relaxed using standard techniques; see~\pref{thm:factored_MDP} for a result with infinite model classes. While realizability can also be relaxed (as in \citet{jiang2017contextual}),
it is impossible to avoid it altogether (that is, to compete with $\OP(\Mcal)$ for arbitrary $\Mcal$) due to exponential lower bounds~\citep{krishnamurthy2016pac}.


\begin{runex*}
For factored MDPs, it is standard to assume the factorization,
formally $\textrm{pa}_i$ for all $i \in [d]$, and the reward function
are known~\citep{kearns1999efficient}. Thus
the natural model class $\Mcal$ is just the set of
all dynamics of the form~\pref{eq:factorization}, which
obey the factorization, with shared reward function. While this class is infinite, our techniques apply as shown in the proof of~\pref{thm:factored_MDP}.
\end{runex*}


\section{Why Model-based RL?}
\label{sec:separation}



This section contains our first main result, 
that model-based methods can be \emph{exponentially} more
sample-efficient than model-free ones.  To our knowledge, this is the
first result of this form.

To show such separation, we must prove a lower bound against all
model-free methods, and, to do so, we first formally define this class
of algorithms. \citet{strehl2006pac} define model-free algorithms to
be those with $o(|\Xcal|^2|\Acal|)$ space, but this definition is specialized to the tabular setting and provides little insight
for algorithms employing function approximation. In contrast, our
definition is information-theoretic: Intuitively, a model-free
algorithm does not operate on the context $x$ directly, but rather
through the evaluations of a state-action function class
$\Gcal$. Formally:


\begin{definition}[Model-free algorithm]
\label{def:model_free}
Given a (finite) function class $\Gcal: (\Xcal\times\Acal) \to \RR$,
define the \emph{$\Gcal$-profile} $\Phi_\Gcal: \Xcal \to
\RR^{|\Gcal|\times|\Acal|}$ by $\Phi_{\Gcal}(x) := [g(x,a)]_{g \in
  \Gcal,a \in \Acal}$. An algorithm is \emph{model-free} using $\Gcal$
if it accesses $x$ exclusively through $\Phi_\Gcal(x)$ for all $x \in
\Xcal$ during its entire execution. 
\end{definition}

In this definition, $\Gcal$ could be a class of $Q$ functions, a class
of policies, or even the union of such classes.  As such, it captures
both value-function-based algorithms like \olive, optimistic
$Q$-learning~\citep{jin2018q}, and Delayed
$Q$-learning~\citep{strehl2006pac} as well as direct policy search algorithms
like policy gradient methods (See~\pref{app:def_value_based} for a
details).%
\footnote{{There are model-free algorithms that elude our definition (for example, ones that approximate the state-action distributions \citep{chen2018scalable, liu2018breaking}), although these algorithms do not address the exploration setting.} }
In~\pref{app:def_value_based}, we show that when $\Gcal$ consists of all $Q$-functions as in the tabular setting,
the underlying context/state
can be recovered from the $\Gcal$-profile, so~\pref{def:model_free}
introduces no restriction whatsoever. 
However, beyond
tabular settings, the $\Gcal$-profile can obfuscate the context from
the agent and may even introduce partial observability. This can lead
to a significant loss of information, which can have a dramatic effect
on the sample complexity. Such information loss is formalized in
the following theorem.



\begin{theorem}
\label{thm:separation}
Fix $\delta, \epsilon\in (0,1]$. There exists a family
  $\Mcal$ of CDPs with horizon $H$, all with the same reward function, and satisfying $|\Mcal| \leq 2^H$,
  such that
\begin{enumerate}[itemsep=0pt,leftmargin=*]
 \item For any CDP in the family, with probability at least
   $1-\delta$, a model based algorithm using $\Mcal$ as the model class (\pref{alg:factored_mdp},
   ~\pref{app:factored_infinite}) outputs $\hat{\pi}$ satisfying
   $v^{\hat{\pi}} \ge v^\star - \epsilon$ using at most
   $\textrm{poly}(H,1/\epsilon,\log(1/\delta))$ trajectories.
 \item With $\Gcal = \texttt{OP}(\Mcal)$, 
   any model-free algorithm using $o(2^H)$ trajectories outputs a
   policy $\hat{\pi}$ with $v^{\hat{\pi}} < v^\star - 1/2$ with
   probability at least $1/3$ on some CDP in the family.
\end{enumerate}
\end{theorem}

See~\pref{app:proof_of_separation} for the proof. Informally, the
result shows that there are CDPs where model-based methods can be exponentially more
sample-efficient than any model-free
method, when given access to a $\Gcal$ satisfying $Q^\star\in\Gcal, \pi^\star \in \Gcal$. As concrete instances of such methods, the lower bound applies to several recent value-based algorithms for CDPs~\citep{krishnamurthy2016pac, jiang2017contextual, dann2018polynomial} as well as any future algorithms developed assuming just realizable optimal value functions or optimal policies. On the other hand, it leaves room for sample efficient model-free techniques that require stronger representation conditions on $\Gcal$. 
To our knowledge, this is the first information-theoretic separation result for any broad class of
model-based/model-free algorithms. Indeed, even the definition of model-free
methods is new here.\footnote{\citet[Section 11.6]{sutton2018reinforcement} have a closely related definition (where the learner can only observe state features), but the definition is specialized to linear function approximation and is subsumed by ours.}

Given this result, it might seem that model-based methods should
always be preferred over model-free ones. However, it is worth also
comparing the assumptions required to enable the two paradigms.  Since $M^\star \in \Mcal$ for each CDP in the family, we
also have $Q^\star \in \OP(\Mcal)$.  This latter \emph{value-function
  realizability} assumption is standard in model-free RL with function
approximation~\citep{antos2008learning,krishnamurthy2016pac}, but our
model-based analog in~\pref{ass:realizable} can be substantially
stronger. As such, model-based methods operating with realizability
typically require more powerful function approximation.
Further, while we view setting $\Gcal = \OP(\Mcal)$ as the most natural choice for the purposes of comparison, using a more expressive $\Gcal$ may reveal state information and circumvent the lower bound (as we show in~\pref{app:overparametrization}).
Thus,
while~\pref{thm:separation} formalizes an argument in favor of
model-based methods, realizability considerations and choice of $\Gcal$ provide important
caveats.

\begin{runex*}
The construction in the proof of~\pref{thm:separation} is a simple
factored MDP with $d=H$, $|\Ocal| = 4$, $|\textrm{pa}_i|=1$ for all
$i$, and deterministic dynamics. As we see, our algorithm has
polynomial sample complexity in all factored MDPs (and a
broad class of other environments).

The construction implies that model-free methods cannot
succeed in factored MDPs. To our knowledge, no information theoretic
lower bounds for factored MDPs exist, but the result does agree
with known computational and representational barriers, namely (a) that planning is
NP-hard~\citep{mundhenk2000complexity}, (b)
that $Q^\star$ and $\pi^\star$ may not
factorize~\citep{guestrin2003efficient},  and (c) that $\pi^\star$ cannot be represented by a polynomially sized
circuit~\citep{allender2003note}. Our result provides a new
form of hardness, namely statistical complexity, for model-free RL in
factored MDPs.
\end{runex*}
\vspace{-0.5cm}

\section{Witnessed Model Misfit}
\label{sec:rank_introduction}
In this section we introduce \emph{witnessed model misfit}, a measure of model error, which we later use to eliminate candidate models in our algorithm.

To verify the validity of a candidate model, a natural idea is to
compare the samples
from the environment with synthetic samples generated from a model $M$.
To formalize this comparison approach, we use Integral Probability
Metrics (IPM) \citep{muller1997integral}: for two probability
distributions $P_1, P_2\in\Delta(\Zcal)$ over $z \in \Zcal$ and a function class $\Fcal:
 \Zcal\to\RR$ that is symmetric (i.e. if $f \in \Fcal$
then $-f \in \Fcal$ also holds), the IPM with respect $\Fcal$ is:
$\sup_{f\in\Fcal} \EE_{z\sim P_1}[f(z)] -\EE_{z\sim P_2}[f(z)]$.
We use IPMs to define 
witnessed model misfit.





\begin{definition}[Witnessed Model Misfit]
For a class $\Fcal: \Xcal\times\Acal\times\Rcal\times\Xcal \to \RR$,
models $M,M' \in \Mcal$ and a time step $h \in [H]$, the
{Witnessed Model Misfit} of $M'$ witnessed by $M$ at level $h$ is:
\begin{align}
\Wcal(M, M', h; \Fcal) \triangleq \sup_{f\in \Fcal} \mathop{\mathbb{E}}_{\overset{\scriptstyle x_h\sim \pi_M}{ a_h\sim \pi_{M'}}} \bigg[  \EE_{(r,x')\sim M'_h}[f(x_h,a_h,r, x')] - \EE_{(r,x')\sim M^\star_{h}}[f(x_h, a_h, r, x')]\bigg],\label{eq:model_error}
\end{align} where for a model $M=(R, P)$, $(r,x')\sim M_h$ is shorthand for $r\sim R_{x_h,a_h}, x'\sim P_{x_h,a_h}$.
\label{def:witnessed_model_misfit}
\end{definition}
$\Wcal(M, M', h; \Fcal)$ is the IPM between two
distributions over $\Xcal\times\Acal\times\Rcal\times\Xcal$ with the
same marginal over $\Xcal\times\Acal$ but two different conditionals
over $(r, x')$, according to $M'$ and the true model $M^\star$, respectively. The marginal over $\Xcal\times \Acal$ is the distribution over context-action pairs when $\pi_M$, the optimal policy of another candidate model $M$, is executed in the true environment.
We call this witnessed
model misfit since $M'$ might successfully masquerade as $M^\star$
unless we find the right context distribution to witness its
discrepancy. Below we illustrate the definition with some examples.




\vspace{-0.3cm}
\begin{ex*}[Total Variation]
When $\Fcal = \{f: \|f\|_{\infty}\leq 1\}$, the witnessed model misfit becomes
\begin{align}
\Wcal(M, M', h; \Fcal) = \mathbb{E}\left[  \variation{R'_{x_h,a_h} \circ P'_{x_h,a_h} - R^\star_{x_h,a_h} \circ P^\star_{x_h,a_h}}  \vert\, x_h\sim \pi_M, a_h\sim \pi_{M'}\right],
\label{eq:tv_model_error}
\end{align}
where $R_{x,a} \circ P_{x,a}$ is the distribution over $\Rcal\times\Xcal$ with $r\sim R_{x,a}, x'\sim P_{x,a}$ independently.
This is just the total variation distance\footnote{We use $\variation{P_1-P_2} = \sum_{x\in\Xcal}\abr{P_1(x) - P_2(x)}$, differing from the standard definition of $\variation{\cdot}$ by a factor of $2$.} between $R'\circ P'$ and $R^\star\circ P^\star$, averaged over context-action pairs $x\sim\pi_{M}, a\sim \pi_{M'}(\cdot|x)$ sampled from the true environment.
\end{ex*}
\vspace{-0.5cm}
\begin{ex*}[Exponential Family]
Suppose the models $M\defeq(R,P)$ are from a conditional exponential
family: conditioned on $(x,a)\in\Xcal\times\Acal$, we have
$R_{x,a}\circ P_{x,a} \defeq \exp\left( \langle \theta_{x,a},
\mathrm{T}(r,x') \rangle \right)/Z(\theta_{x,a})$ for parameters
$\theta_{x,a} \in \Theta\defeq \{\theta: \nbr{\theta} \leq 1\}\subset
\RR^m$ with partition function $Z(\theta_{x,a})$ and sufficient
statistics $\mathrm{T}:\Rcal\times\Xcal\to \RR^{m}$. With $\Vcal =
\{\Xcal\times\Acal \to \Theta\}$, we design $\Fcal = \{
(x,a,r,x')\mapsto \langle v(x,a), \mathrm{T}(r,x') \rangle: v\in\Vcal
\}$. In this setting, witnessed model misfit is
\begin{align*}
\Wcal(M, M', h; \Fcal) = \EE_{x_h\sim \pi_M, a_h\sim \pi_{M'}}\bigg[   \Big\|  \EE_{(r,x')\sim M'_{h}}[\mathrm{T}(r,x')]  - \EE_{(r,x')\sim M^\star_{h}}[\mathrm{T}(r,x')]   \Big\|_{\star}        \bigg],
\end{align*}
with  $\nbr{x}_{\star} \defeq \sup\{ \langle x, \theta \rangle \vert \nbr{\theta} \leq 1\}$.
Here,
we measure distance, in the dual norm, between the expected sufficient statistics of
$(r,x')$ sampled from $M'$ and the true model $M^\star$. See~\pref{app:exponential_family}.
\end{ex*}
\vspace{-0.1cm}
\vspace{-0.6cm}
\begin{ex*}[MMD]
When $\Fcal$ is a unit ball in an RKHS, we obtain MMD~\citep{gretton2012kernel}.
\end{ex*}
\vspace{-0.2cm}

Witnessed model misfit is also closely related to the \emph{average
  Bellman error}, introduced by~\cite{jiang2017contextual}.  Given
$Q$ functions, $Q$ and $Q'$, the average Bellman error at time step $h$ is:
\begin{align}
&\mathcal{E}_{B}(Q, {Q}', h) \defeq \EE \left[ Q'(x_h, a_h) - r_h - Q'(x_{h+1}, a_{h+1}) \,\big\vert\, x_h\sim \pi_{Q}, a_{h:h+1}\sim \pi_{Q'} \right] \label{eq:bellman_error}
\end{align}
where $\pi_Q$ is the greedy policy associated with $Q$, i.e.,
$\pi_Q(a|x) \defeq \one\{a = \argmax_{a'}Q(x,a')\}$, and the random trajectories (w.r.t.~which we take expectation) are generated in the true environment $M^\star$.



When the $Q$ functions are derived from a model class, meaning
that $\Qcal = \texttt{OP}(\mathcal{M})$, we can extend the
definition to any pair of models $M, M' \in \Mcal$, using $Q_M$ and
$Q_{M'}$.
In such cases, the average Bellman error for $M,M'$ is just the model misfit witnessed by the function $f_{M'}(x,a,r,x') = r + V_{M'}(x')$.
We conclude this section with an assumption about the class $\Fcal$.
\begin{assum}[Bellman domination using $\Fcal$]
\label{ass:test_function_realizable}
$\Fcal$ is symmetric, finite in size,\footnote{As before, our results apply whenever $\Fcal$ has bounded statistical complexity. We describe a more complicated algorithm with no dependence on the complexity of $\Fcal$ in the appendix.}
$\forall f \in \Fcal:\,\|f\|_\infty \leq 2$, and the witnessed model misfit~\pref{eq:model_error} satisfies $\forall M,M'\in \Mcal:\,\,\Wcal(M, M',h; \Fcal) \geq \mathcal{E}_B(Q_M, Q_{M'},h)$.
\end{assum}
As discussed above, one easy way to satisfy this assumption is to
ensure that the special functions $r + V_M(x')$ are contained in
$\Fcal$ for all $M \in \Mcal$, but this is not the only way as we will see
in~\pref{sec:factored_MDP}.\footnote{We allow the $\ell_\infty$ bound
  of 2 to accommodate these functions whose range can be 2 under our
  assumptions.} The Bellman domination condition
in~\pref{ass:test_function_realizable} plays an important role in the
algorithm we present next, as it allows us to detect the
suboptimality of a model in terms of the value attained by its optimal
policy in the actual MDP.

\section{A Model-based Algorithm}
\label{sec:algorithm}

In this section, we present our main algorithm and sample complexity
results.  We start by describing the algorithm. Then, working towards a
statistical analysis, we introduce the \emph{witness rank}, a new
structural complexity measure. We end this section with the main
sample complexity bounds.

\subsection{Algorithm}
Since we do not have access to $M^\star$, we must estimate the
witnessed model misfit from samples. Since $\Fcal$ will always be clear from the context, we drop it from the arguments to the model misfit for succinctness.
Given a dataset $\Dcal \defeq \{(x_h^{(n)}, a_h^{(n)}, r^{(n)}_h, x^{(n)}_{h+1})\}_{n=1}^N$ with
\begin{align*}
x_h^{(n)}\sim \pi_{M},  ~a_h^{(n)} \sim U(\Acal),~ (r_h^{(n)},x_{h+1}^{(n)})\sim M^\star_h, 
\end{align*}
denote the importance weight $\rho^{(n)} \triangleq K
\pi_{M'}(a_h^{(n)}|x_{h}^{(n)})$. We simply use the empirical model
misfit:
\begin{align}
\tWcal(M, M', h) \triangleq \max_{f\in\Fcal} \sum_{n=1}^N  \frac{\rho^{(n)}}{N}  \EE_{(r,x')\sim M'_{h}} \sbr{f(x_h^{(n)}, a_h^{(n)},  r, x') -    f(x_h^{(n)}, a_h^{(n)}, r_h^{(n)}, x_{h+1}^{(n)})} .
\label{eq:model_err_est_2}
\end{align}
Here the importance weight $\rho^{(n)}$ accounts for distribution
mismatch, since we are sampling from $U(\Acal)$ instead of
$\pi_{M'}$. Via standard uniform convergence arguments
(in~\pref{app:proof_of_sample_complexity}) we show that
$\tWcal(M,M',h)$ provides a high-quality estimate of $\Wcal(M,M',h)$
under~\pref{ass:test_function_realizable}.


We also require an estimator for the average Bellman error
$\Ecal_B(M,M,h)$.
Given a data set $\{(x^{(h)}_n, a^{(n)}_h, r_h^{(i)},
x^{(n)}_{h+1})\}_{n=1}^N$ where
$x^{(n)}_h\sim \pi_M, a^{(n)}_h\sim \pi_M$, and $(r_h^{(n)},x^{(n)}_{h+1})\sim M^{\star}_h$,  
we
form an {unbiased}  estimate of $\Ecal_{B}(M, M, h)$ as
\begin{align}
{\hEcal}_{B}(M, M, h) \triangleq \frac{1}{N}\sum_{n=1}^N \sbr{
  Q_M(x_h^{(n)},a_h^{(n)}) - \sbr{r_h^{(n)}+V_{M}(x^{(n)}_{h+1})}}.
\label{eq:empirical_bellman_error}
\end{align}

\begin{algorithm}[t]
\begin{algorithmic}[1]
\STATE $\Mcal_0 = \Mcal$
\FOR {$t = 1, 2, ...$}
	\STATE Choose model optimistically: $M^t = \argmax_{M\in\Mcal_{t-1}} v_{{M}}$,  set $\pi^t = \pi_{M^t}$ \label{line:opt}
	\STATE Execute $\pi^{t}$ to collect $n_e$ trajectories $\{(x_h^{i}, a_h^{i}, r_h^{i})_{h=1}^H\}_{i=1}^{n_e}$ and set $\hat{v}^{\pi^{t}} = \tfrac{1}{n_e}\sum_{i=1}^{n_e} \big(\sum_{h=1}^H r_h^{i}\big)$
\vspace{-0.4cm}
    \STATE \textbf{if} $\lvert \hat{v}^{\pi^{t}} - v_{M^t} \rvert \leq \epsilon/2$ \textbf{then} Terminate and output $\pi^t$ \textbf{end if}
\smallskip
	\STATE Find $h_t$ such that ${\hEcal}_{B}(M^t, M^t, h_t) \geq \frac{\epsilon}{4H}$ (See~\pref{eq:empirical_bellman_error}) \label{line:select_h_t}

	\STATE Collect trajectories $\{(x^{(i)}_h,a^{(i)}_h, r_h^{(i)})_{h=1}^{H}\}_{i=1}^n$ where $a_{h}^{(i)} \sim \pi^t$ for $h\neq h_t$ and $a_{h_t}^{(i)}\sim U(\mathcal{A})$ \label{line:gen_data_Wcal}
\smallskip
	\STATE \textbf{for} $M' \in \Mcal_{t-1}$ \textbf{do}\,\,
	 Compute ${\tWcal}(M^t, M', h_t)$  
(See~\pref{eq:model_err_est_2}) \label{line:estimate_model_error}\,\,
	\textbf{end for}
	\STATE Set $\Mcal_{t} = \{M\in \Mcal_{t-1}:  {\tWcal}(M^t, M, h_t) \leq \phi  \}$  \label{line:eliminate}
\ENDFOR
\end{algorithmic}
\caption{Inputs: ($\Mcal, \Fcal, n, n_e, \epsilon, \delta, \phi$)}
\label{alg:main_alg}
\end{algorithm}


The pseudocode is displayed in~\pref{alg:main_alg}. The algorithm is
round-based, maintaining a version space of models and eliminating a
model from the version space when the discrepancy between
the model and the ground truth $M^\star$ is witnessed. The witness distributions
are selected using a form of \emph{optimism}: at each round, we
select, from all surviving models, the one with the highest predicted
value, and use the associated policy for data collection. If the
policy achieves a high value in the environment, we simply return
it. Otherwise we estimate the witnessed model misfit on the context
distributions induced by the policy, and we shrink the version space
by eliminating all incorrect models. Then we proceed to the next
iteration.

Intuitively, using a simulation lemma analogous to Lemma 1 of~\citet{jiang2017contextual}, if $M^t$ is the optimistic model at round $t$ and we do
not terminate, then there must exist a time step $h_t$
(\pref{line:select_h_t}) where the average Bellman error is
large. Using~\pref{ass:test_function_realizable}, this also implies
that the witness model misfit for $M^t$ witnessed by $M^t$ itself must
be large. Thus, if $t$ is a non-terminal round, we ensure that $M^t$
and potentially many other models are eliminated.



The algorithm is similar to \olive, which
uses average Bellman error instead of witnessed model misfit to shrink
the version space. However, by appealing
to~\pref{ass:test_function_realizable}, witness model misfit provides
a more aggressive elimination criterion, since a large average Bellman error on a distribution immediately implies a large witnessed model misfit on the same distribution, but the converse does not necessarily hold. Since the algorithm
uses an aggressive elimination rule, it often requires
fewer iterations than \olive, as discussed below.

\paragraph{Computational considerations.}
In this work, we focus on the sample complexity of model-based RL,
and~\pref{alg:main_alg}, as stated, admits no obvious efficient
implementation. The main bottleneck, for efficiency, is the optimistic
computation of the next model in~\pref{line:opt} where we perform a
constrained optimization, restricted to the class of all models not
eliminated so far. The objective in this problem encapsulates a
planning oracle to map models from our class to their values, and the
constraints involve enforcing small values of witness model misfit on
the prior context distributions. While the witness model misfit is
linear in the transition dynamics, finding an optimistic value
function induces bilinear, non-convex constraints even in a tabular
setting. This resembles known computational difficulties with \olive,
but we note that the recent hardness result
of~\citet{dann2018polynomial} for \olive does not apply
to~\pref{alg:main_alg}, leaving the possibility of an efficient
implementation open.


\subsection{A structural complexity measure}

So far, we have imposed realizability and expressivity
assumptions
(\pref{ass:realizable} and~\pref{ass:test_function_realizable})  on $\Mcal$ and $\Fcal$. 
Unfortunately, these alone do not enable tractable reinforcement learning with
polynomial sample complexity, as verified by the following simple
lower bound.
\begin{proposition}
\label{prop:lower_bound_realizable}
Fix $H, K\in\mathbb{N}^+$ with $K\geq 2$ and $\epsilon \in (0,
\sqrt{1/8})$. There exists a family of MDPs, classes $\Mcal,\Fcal$ satisfying~\pref{ass:realizable} and~\pref{ass:test_function_realizable} for all MDPs in the family with $|\Mcal| = |\Fcal| = K^{H-1}$,
and a constant $c>0$, such that the following holds:
For any algorithm that takes $\Mcal$, $\Fcal$ as inputs and uses $T \leq cK^{H-1}/\epsilon^2$ episodes, 
the algorithm
outputs a policy $\hat{\pi}$ with $v^{\hat{\pi}} < v^\star - \epsilon$
with probability at least $1/3$ for some MDP in the family.
\end{proposition}
The proof, provided in~\pref{app:proof_of_lower_bound_realizable},
adapts a construction from~\citet{krishnamurthy2016pac} for showing that
value-based realizability is insufficient for model-free algorithms.
The result suggests that we must introduce further structure to obtain
polynomial sample complexity guarantees. We do so with a new
structural complexity measure, the \emph{witness rank}.


For any matrix $B\in\mathbb{R}^{n\times n}$, define
$\textrm{rank}(B,\beta)$ to be the smallest integer $k$ such that $B =
UV^\top$ with $U,V\in\mathbb{R}^{n\times k}$ and for every pair of
rows $u_i,v_j$, we have $\|u_i\|_2\cdot\|v_j\|_2 \leq \beta$. This
generalizes the standard definition of matrix rank, with a condition
on the row norms of the factorization.

\begin{definition}[Witness Rank]
\label{def:refined_rank}
Given a model class $\Mcal$, test functions $\Fcal$, and $\kappa \in
(0,1]$, for $h\in [H]$, define the set of matrices $\Ncal_{\kappa, h}$
  such that any matrix $A\in \Ncal_{\kappa, h}$ satisfies:
\begin{align*}
A\in\mathbb{R}^{\abr{\Mcal}\times\abr{\Mcal}}, \;\; \kappa \Ecal_{B}(M, M', h) \leq A(M, M') \leq \Wcal(M, M', h), \forall M, M' \in \Mcal,
\end{align*}
We define the witness rank
as
\begin{align*}
\wrank(\kappa, \beta, \Mcal,\Fcal, h) \defeq \min_{A\in \Ncal_{\kappa, h}} \text{rank}(A,\beta).
\end{align*}
\end{definition}

We typically suppress the dependence on $\beta$ because it appears
only logarithmically in our sample complexity bounds. Any $\beta$ that
is polynomial in other parameters ($K,H$, and the rank itself)
suffices.

To build intuition for the definition, first consider the extreme
where $A(M,M') = \Wcal(M,M',h)$. The rank of this
matrix corresponds to the number of context distributions required to
verify non-zero witnessed model misfit for all incorrect models. This
follows from the fact that there are at most $\rank(\Wcal)$ linearly
independent \emph{rows} (context distributions), so any non-zero
column (an incorrect model) must have a non-zero in at least one of
these rows. 
Algorithmically, if we can find the policies $\pi_M$
corresponding to these rows, we can eliminate \emph{all} incorrect
models to find $M^\star$ and hence $\pi^\star$.

At the other extreme, we have $A(M,M') = \kappa\Ecal_B(M,M',h)$,
the \emph{Bellman error matrix} 
. The rank of this matrix, called \emph{Bellman rank}, provides an upper bound on the witness rank by construction, and is known to be small for many natural RL settings, including tabular and low-rank MDPs, reactive POMDPs, and reactive PSRs (see Section 2 of~\citet{jiang2017contextual} for details). The minimum over all sandwiched $A$ matrices in the definition of the witness rank allows a smooth interpolation between these extremes in general. We further note that the choice of the class $\Fcal$ defining the IPM also affects the witness model misfit and hence the witness rank. Adapting this class to the problem structure yields another useful knob to control the witness rank, as we show for the running example of factored MDPs in~\pref{sec:factored_MDP}.
%
\subsection{Sample complexity results}

We now present a sample complexity guarantee for~\pref{alg:main_alg}
using the witness rank.  Denote $\wrank_{\kappa} \defeq \max_{h\in
  [H]} \wrank(\kappa, \beta, \Mcal, \Fcal, h)$. The main guarantee is
the following theorem.

\begin{theorem}
\label{thm:refined_guarantee}
Under~\pref{ass:realizable} and~\pref{ass:test_function_realizable}, for any $\epsilon,\delta,\kappa \in (0,1]$, set $\phi = \frac{\kappa\epsilon}{48 H\sqrt{\wrank_\kappa}}$, and denote $T = H \wrank_{\kappa}\log(\beta/2\phi)/\log(5/3)$. Run~\pref{alg:main_alg} with inputs $(\Mcal, \Fcal,n_e, n, \epsilon ,\delta, \phi)$, where $n_e = \Theta\big(H^2 \log(HT/\delta)/\epsilon^2\big)$ and $n = \Theta\big( H^2 K \wrank_\kappa \log(T|\Mcal||\Fcal|/\delta)/(\kappa\epsilon)^2\big)$.
Then with probability at least $1-\delta$,~\pref{alg:main_alg} outputs a policy $\pi$ such that $v^{\pi}\geq v^{\star} - \epsilon$. The number of trajectories collected is at most
$\tilde{O}\left(\frac{H^3  \wrank_{\kappa}^2 K }{\kappa^2\epsilon^2} \log \left(\frac{T\abr{\Fcal}\abr{\Mcal}}{\delta}\right)\right)$.
\end{theorem}
The proof is included in~\pref{app:proof_of_sample_complexity}. Since,
as we have discussed, many popular RL models admit low Bellman rank
and hence low witness rank,~\pref{thm:refined_guarantee} verifies
that~\pref{alg:main_alg} has polynomial sample complexity in all of
these settings. A noteworthy case that does not have small Bellman
rank but does have small witness rank is the factored MDP, which we
discuss further in~\pref{sec:factored_MDP}.

\paragraph{Comparison with \olive.}
The minimum sample complexity is achieved at
$\inf_{\kappa} \wrank_\kappa/\kappa$, which is never larger than  the
Bellman rank.
In fact
when $\kappa=1$, the sample complexity bounds match in all terms
except (a) we replace Bellman rank with witness rank, and (b) we have
a dependence on model and test-function complexity $\log
(|\Mcal||\Fcal|)$ instead of $Q$-function complexity $\log
|\OP(\Mcal)|$. The witness rank is never larger than the Bellman rank
and it can be substantially smaller, which is favorable for
~\pref{alg:main_alg}. However, we always have $\log |\Mcal| \geq \log
|\OP(\Mcal)|$ and since we require realizability, the model class can
be much larger than the induced $Q$-function class. Thus the two results
are in general incomparable, but for problems where modeling the
environment is not much harder than modeling the optimal $Q$-function
(in other words $\log (|\Mcal| |\Fcal|) \approx \log |\OP(\Mcal)|$),
~\pref{alg:main_alg} can be substantially more sample-efficient than
\olive. 

\paragraph{Adapting to unknown witness rank.} In~\pref{thm:refined_guarantee}, the algorithm needs to know the
value of $\kappa$ and $\wrank_{\kappa}$, 
as they are used to determine $\phi$ and
$n$. In~\pref{app:uknown_rank}, we show that a standard doubling trick
can adapt to unknown $\kappa$ and $\wrank_{\kappa}$. The sample
complexity for this adaptation is given by\\ $\tilde{O}(H^3 \wrank_{\kappa^\star}^2 K /
((\kappa^\star\epsilon)^2)\log(|\Mcal||\Fcal|/\delta))$, where
$\kappa^\star \defeq \arg\min_{\kappa\in (0,1]} \wrank_{\kappa}/ \kappa$
  minimizes the bound in~\pref{thm:refined_guarantee}.
  A similar technique was used to adapt \olive
  to handle unknown Bellman rank.



\paragraph{Extension to infinite $\Mcal$.}
\pref{thm:refined_guarantee} as stated assumes that $\Mcal$ and
$\Fcal$ are finite classes. It is desirable to allow rich classes
$\Mcal$ to have a better chance of satisfying realizability of $M^\star$ in~\pref{ass:realizable}.
Indeed, it is possible to use
standard covering arguments to handle the case of infinite $\Mcal$, 
and we demonstrate this in the context of factored MDPs in \pref{thm:factored_MDP}. 

\paragraph{Extension to infinite $\Fcal$.} While our result also extends to
infinite $\Fcal$ with bounded statistical complexity, 
it is desirable to handle even richer classes, 
for example, 
$\Fcal = \{f: \nbr{f}_{\infty} \leq 1\}$ for the total variation distance, which does not admit uniform convergence. 
To handle such rich
classes, we borrow ideas from the Scheff\'e tournament
of \citet{devroye2012combinatorial},\footnote{The classical Scheff\'e
  tournament targets the following problem: given a set of
  distributions $\{P_i\}_{i=1}^K$ over $\Xcal$, and a set of i.i.d
  samples $\cbr{x_i}_{i=1}^N$ from $P^\star \in \Delta(\Xcal)$,
  approximate the minimizer $\argmin_{i\in [K]} \variation{P_i -
    P^\star}$. } and extend the method to handle conditional
distributions and IPMs induced by an arbitrary class.  The analysis
here covers the total-variation based witnessed model misfit defined
in~\pref{eq:tv_model_error} as a special case.

\begin{theorem}
\label{thm:tv_sample_complexity}
Under~\pref{ass:realizable} and~\pref{ass:test_function_realizable},
but with no restriction on size of $\Fcal$,\footnote{In fact,~\pref{ass:test_function_realizable} holds automatically if we choose $\Fcal = \{f: \|f\|_\infty\le 2\}$.} there exists an algorithm such that:
For any $\epsilon,\delta \in (0,1]$,
with probability at least $1-\delta$ the algorithm outputs a policy
$\pi$ such that $v^\pi \geq v^\star - \epsilon$ with at most
$\tilde{O}\left(\frac{H^3  \wrank_{\kappa}^2 K }{\kappa^2\epsilon^2} \log \left(\frac{T\abr{\Mcal}}{\delta}\right)\right)$ trajectories collected,
where $T = H \wrank_{\kappa}\log(\beta/2\phi)/\log(5/3)$.
\end{theorem}

The algorithm modifies~\pref{alg:main_alg} to incorporate
the Scheff\'e estimator instead of the direct empirical estimate for the witnessed model misfit~\pref{eq:model_err_est_2}. We defer
the details of the algorithm
and analysis to~\pref{app:tv_results}.  The
main improvement over~\pref{thm:refined_guarantee} is that the sample
complexity here has no dependence on $\Fcal$, so we may use test
function classes with unbounded statistical complexity.

\section{Case Study on MDPs with Factored Transitions}
\label{sec:factored_MDP}
In this section, we study the running example of factored MDPs in detail.
Recall the definition of factored transition dynamics
in~\pref{eq:factorization}. Following~\citet{kearns1999efficient}, we assume 
$R^\star$ and $\{\parent_i\}$ are known, and $\Mcal$ is the continuous space of \emph{all} models obeying the factored transition structure and with $R^\star$ as the reward function. 
For this setting, we have
the following guarantee.

\begin{theorem}
\label{thm:factored_MDP}
For MDPs with factored transitions and for any $\epsilon,\delta \in
(0,1]$, with probability at least $1-\delta$ a modification of~\pref{alg:main_alg} (\pref{alg:factored_mdp} in~\pref{app:factored_infinite})
  outputs a policy $\pi$ with $v^\pi \geq v^\star - \epsilon$ using at
  most $\tilde{O}(d^2 L^3 H K^2 \log (1/\delta)/\epsilon^2)$
  trajectories.
\end{theorem}

This result should be contrasted with the $\Omega(2^H)$ lower bound
from~\pref{thm:separation} that actually applies precisely to  this
setting, where the lower bound construction has description length $L$ polynomial in $H$ (see \pref{app:proof_of_separation} for details). Combining the two results we have demonstrated exponential
separation between model-based and model-free algorithms for MDPs with
factored transitions.


Comparing with~\pref{thm:refined_guarantee}, the main improvement is
that we are working with an infinite model class of all possible
factored transition operators. The linear scaling with $H$, which seems to be an improvement, is purely cosmetic as we have $L =
\Omega(H)$ here.
\pref{thm:factored_MDP} involves a slight modification
to~\pref{alg:main_alg}, in that it uses a slightly different notion of witnessed
model misfit,
\begin{align}\label{eq:model_error_factored_mdp}
\Wcal_{F}(M, M', h) = \max_{f\in\Fcal} \,\, \mathop{\E}_{\overset{\scriptstyle x_h\sim \pi_M}{ a_h\sim U(\Acal)}} \sbr{  \mathop{\EE}_{(r,x')\sim M'_h}[f(x_h,a_h,r, x')] - \mathop{\EE}_{(r,x')\sim M^\star_{h}}[f(x_h, a_h, r, x')]}.
\end{align}
together with an $\Fcal$ specially designed for factored MDPs (subscript of $\Wcal_F$ indicates adaptation to factored MDPs).
The main difference with~\pref{eq:tv_model_error} is that $a_h$ is
sampled from $U(\Acal)$ rather than $\pi_{M'}$. This modification is
crucial to obtain a low witness rank, since $\pi_{M'}$ is in general
not guaranteed to be factored (recall the representation hardness discussed at the end of~\pref{sec:separation}). Thanks to uniformly random actions and our choice of $\Fcal$, $\Wcal_F$ essentially computes the sum of the TV-distances across all
factors, and the corresponding matrix naturally factorizes and yields low witness rank. On the other hand, the choice of $\pi_{M'}$ for the general case allows a direct comparison with Bellman rank and leads to better guarantees in general, so we do not use the definition~\eqref{eq:model_error_factored_mdp} more generally. We defer the details of the algorithm and its analysis to \pref{app:factored_infinite}. 
\section{Related Work}
\label{sec:related}


For tabular MDPs, a number of sample-efficient RL approaches
exist, mostly model-based~\citep{kearns2002near, jaksch2010near, dann2015sample,
  szita2010model, azar2017minimax}, but some are model-free~\citep{strehl2006pac, jin2018q}. 
In contrast, our work focuses on more realistic rich-observation
settings.\footnote{In fact, our information-theoretic definition of
  model-free methods (\pref{def:model_free}) is uninteresting in the
  tabular setting.} For factored MDPs, all prior sample-efficient
algorithms are
model-based~\citep{kearns1999efficient,osband2014near}. 
With rich observations, many prior works either focus on structured control settings like LQRs~\citep{abbasi2011regret,dean2018regret} or Lipschitz-continuous MDPs~\citep{kakade2003exploration, ortner2012online, pazis2013pac,
  lakshmanan2015improved}. In LQRs,~\citet{tu2018gap} show a gap between model-based and a particular model-free algorithm, but not an algorithm agnostic lower bound, as we show here for factored MDPs. We expect that our algorithm or natural variants are sample-efficient in many of these specific
settings.

In more abstract settings, most sample-efficient algorithms are
model-free~\citep{wen2013efficient,krishnamurthy2016pac,jiang2017contextual,dann2018polynomial}.
Our work can be seen as a model-based analog to
\citet{jiang2017contextual}, which among the above references, studies
the most general class of environments. 



On the model-based side,~\citet{lattimore2013sample}
and~\citet{osband2014model} obtain sample complexity guarantees; the
former makes no assumptions but the guarantee scales linearly with the
model class size, and the latter makes continuity assumptions, so both
results have more limited scope than ours. \citet{ok2018exploration} propose a complexity measure for structured RL problems, but their results are for asymptotic regret in tabular or Lipschitz MDPs. 

On the empirical side, models are often used to speed up learning~\citep[see e.g.,][for classical references in robotics]{aboaf1989task,deisenroth2011learning}.
Such results provide empirical evidence that models can be
statistically valuable, which complement our theoretical results.

Finally, two recent papers share
some technical similarities to our work.
\citet{farahmand2017value} also use IPMs to detect model error
but their analysis is restricted to test functions that form a ball in
an RKHS, and they do not address exploration issues. \citet{xu2018algorithmic} devise a model-based algorithm with function approximation, 
but their algorithm performs local policy
improvement and cannot find a globally optimal policy in a
sample-efficient manner. 

\section{Discussion}
\label{sec:discussion}
We study model-based RL in general
contextual decision processes. We derive an algorithm for general CDPs
and prove that it has sample complexity upper-bounded by a new
structural notion called the witness rank, which is small in many
settings of interest. Comparing model-based and model-free methods, we
show that the former can be exponentially more sample
efficient in some settings, but they also require stronger
function-approximation capabilities, which can result in worse sample
complexity in other cases. Comparing the guarantees here with those
derived by~\citet{jiang2017contextual} precisely quantifies these
tradeoffs, which we hope guides future design of RL algorithms.


\bibliography{refs}

\appendix

\setlength{\abovedisplayskip}{10pt}
\setlength{\belowdisplayskip}{10pt}

\section{Proof of~\pref{thm:refined_guarantee}}
\label{app:proof_of_sample_complexity}

We first present several lemmas that are useful for proving~\pref{thm:refined_guarantee}.

\begin{fact}
\label{fact:bellman_error_model}
For any two models $ M, {M}'$, the corresponding average Bellman error can be written as
\begin{align}
&\mathcal{E}_{B}({M},{M}',  h) \triangleq  \mathcal{E}_{B}(Q_M, Q_{M'}, h) \nonumber\\
 =& \mathbb{E}_{x_h\sim \pi_M,a_h\sim\pi_{M'}}\sbr{ \mathbb{E}_{(r,x')\sim M'_{h}} \sbr{r+V_{M'}(x')}  - \mathbb{E}_{(r,x')\sim M^\star_{h}}\sbr{r+V_{M'}(x')}  }.
\label{eq:Bellman_error_model_version}
\end{align}
\end{fact}

\begin{lemma}
[Lemma 11 of \cite{jiang2017contextual}]
\label{lem:volume_shrink}
Consider a closed and bounded set $V\subset \mathbb{R}^d$ and a vector $p\in\mathbb{R}^d$. Let $B$ be any origin-centered enclosing ellipsoid of $V$. Suppose there exists $v\in V$ such that $p^{\top} v \geq \kappa$ and define $B_{+}$ as the minimum volume enclosing ellipsoid of $\{v\in B: |p^{\top} v | \leq \frac{\kappa}{3\sqrt{d}}\}$. With $\text{vol}(\cdot)$ denoting the (Lebesgue) volume, we have:
\begin{align}
\frac{\text{vol}(B_+)}{\text{vol}(B)} \leq \frac{3}{5}. \nonumber
\end{align}
\end{lemma}

Recall that $V_M, \pi_M$ denote the optimal value function and policy
derived from model $M$, and that $v_M$ denotes $\pi_M$'s value in
$M$. For any policy $\pi$, $v^{\pi}$ denotes the policy $\pi$'s value
in the true environment.
\begin{lemma}
[Simulation Lemma]
\label{lem:simulation_lemma}
Fix a model $M$. Under~\pref{ass:test_function_realizable}, we have
\begin{align*}
v_{M} - v^{\pi_{M}} =  \sum_{h=1}^H \mathcal{E}_{B}( M,  M, h), \qquad \textrm{ and } \qquad
 v_{M} - v^{\pi_{M}} \leq  \sum_{h=1}^H \mathcal{W}( M, M, h).
\end{align*}
\end{lemma}
\begin{proof}
Start at time step $h=1$,
\begin{align*}
&\mathbb{E}_{x_1\sim P_0} [V_{M}(x_1) - V^{\pi_{M}}(x_1)] \\
& =    \mathbb{E}_{x_1\sim P_0, a_1\sim \pi_{M}} \left[\mathbb{E}_{(r, x_2) \sim M_{x_1,a_1}} \left[r +  V_{M}(x_2)\right] - \mathbb{E}_{(r,x_2)\sim M^\star_{ x_1,a_1}}\left[r + V^{\pi_{M}}(x_2) \right]\right]  \\
& =\mathbb{E}_{x_1\sim P_0, a_1\sim \pi_{M}} \left[\mathbb{E}_{(r,x_2)\sim M_{x_1,a_1}} \left[ r + V_{ M}(x_2)\right] - \mathbb{E}_{(r,x_2)\sim M^\star_{x_1,a_1}}\left[ r + V_{M}(x_2)\right] \right] \\
&\;\;\;\;\;\;\; +  \mathbb{E}_{x_1\sim P_0, a_1\sim \pi_{M}}\left[\mathbb{E}_{(r,x_2)\sim M^\star_{x_1,a_1}}\left[ V_{M}(x_2)\right] -\mathbb{E}_{(r,x_2)\sim M^\star_{x_1,a_1}}\left[V^{\pi_{M}}(x_2) \right]\right],
\end{align*}
where the first equality is based on applying Bellman's equation to
$V_M$ in $M$ and $V^{\pi_M}$ in $M^\star$. 
Now,
by~\pref{fact:bellman_error_model}, the first term above is exactly
$\Ecal_B(M,M,1)$.
The second term can be expressed as,
\begin{align*}
\mathbb{E} \left[V_{M}(x_2) - V^{\pi_{M}}(x_2) \vert x_2 \sim \pi_M \right],
\end{align*}
which we can further expand by applying the same argument recursively to obtain the identity involving the average Bellman errors.
For the bound involving the witness model misfit, since $V_M \in \Fcal$, we simply observe that $\Ecal_B(M,M,h) \leq \Wcal(M,M,h)$.
\end{proof}

Next, we present several concentration results.
\begin{lemma}
\label{lem:deviation_hat_v}
Fix a policy $\pi$, and fix $\epsilon,\delta \in (0,1)$. Sample $n_e = \frac{\log(2/\delta)}{(2\epsilon)^2}$ trajectories $\{(x_h^{(i)}, a_h^{(i)}, r_h^{(i)})_{h=1}^H\}_{i=1}^{n_e}$ by executing $\pi$ and set $\hat{v}^{\pi} = \frac{1}{n_e } \sum_{i=1}^{n_e}\sum_{h=1}^H r_h^{(i)}$. With probability at least $1-\delta$, we have $\left\lvert \hat{v}^{\pi} - v^{\pi} \right\rvert \leq \epsilon$. 
\end{lemma}
The proof is a direct application of Hoeffding's inequality on the random variables $\sum_{h=1}^H r_{h}^{(i)}$. 

Recall the definitions of ${\tWcal}$ and ${\hEcal}_B$
from~\pref{eq:model_err_est_2} and~\pref{eq:empirical_bellman_error},
and the shorthand notation $(r,x') \sim M_h$, which stands for $r \sim R_{x_h, a_h}$ and $x' \sim P_{x_h, a_h}$ (with $(R, P) = M$) whenever the identities of $x_h$ and $a_h$ are clear from context.
\begin{lemma}[Deviation Bound for $\hEcal_{M}$] Fix $h$ and model $M\in\Mcal$. Sample a
 \label{lem:deviation_bound}
dataset\\ $\Dcal = \cbr{(x_h^{(i)},a_h^{(i)},r_h^{(i)}, x_{h+1}^{(i)})}_{i=1}^N$
with $x_h^{(i)}\sim\pi_M,a_h^{(i)}\sim U(\Acal), (r_h^{(i)}, x_{h+1}^{(i)})\sim
M^\star_{h}$ of size $N$. Then with probability at
least $1-\delta$, we have for all $M' \in \Mcal$:
\begin{align*}
\abr{{\tWcal}(M, M', h) - \Wcal(M, M', h)} \leq    \sqrt{\frac{2K\log(2|\Mcal||\Fcal|/\delta)}{N}} + \frac{2K\log(2|\Mcal||\Fcal|/\delta)}{3N}.
\end{align*}
\end{lemma}
\begin{proof}

Fix $M' \in \Mcal$ and $f\in\Fcal$, define the random variable $z_i(M',f)$ as:
\begin{align*}
z_i(M',f) \defeq K\pi_{M'}(a_h^{(i)}|x_{h}^{(i)})\left( \EE_{(r,x')\sim M'_{h}} f(x_h^{(i)}, a_h^{(i)}, r, x')  - f(x_{h}^{(i)}, a_{h}^{(i)}, r_h^{(i)}, x_{h+1}^{(i)}) \right).
\end{align*}
The expectation of $z_i(M',f)$ is
\begin{align*}
\EE[z_i(M',f)] = \underbrace{\EE_{x_h\sim\pi_M,a_h\sim\pi_{M'}}[\EE_{(r,x')\sim
M'_{h}}[f(x_h,a_h,r, x')] - \EE_{(r,x')\sim M^\star_h}[f(x_h,a_h,r, x')]]}_{\defeq d(M',M^\star,f)}, 
\end{align*}
and it is easy to verify that $\mathrm{Var}(z_i(M',f)) \leq 4K$. Hence, we can
apply Bernstein's inequality, so that with probability at least
$1-\delta$, we have
\begin{align*}
&\abr{\frac{1}{N}\sum_{i=1}^N z_i(M',f)   -  d(M',M^\star,f)} \leq \sqrt{\frac{2K\log(2/\delta)}{N}} + \frac{2K\log(2/\delta)}{3N}.
\end{align*}
Via a union bound over $\Mcal$ and $\Fcal$, we have that for all pairs
$M'\in\Mcal, f\in\Fcal$, with probability at least $1-\delta$:
\begin{align}
\label{eq:concentration_model}
\abr{\frac{1}{N}\sum_{i=1}^N z_i(M',f)   -  d(M',M^\star,f)} \leq \sqrt{\frac{2K\log(2|\Mcal||\Fcal|/\delta)}{N}} + \frac{2K\log(2|\Mcal||\Fcal|/\delta)}{3N}.
\end{align}
For fixed $M'$, we have shown uniform convergence over $\Fcal$, and
this implies that the empirical and the population maxima must be
similarly close, which yields the result.
\end{proof}

\begin{lemma}[Deviation Bound on $\hEcal_B$]
Fix model $M\in\Mcal$. Sample a
dataset $\Dcal = \cbr{(x_h^{(i)},a_h^{(i)},r_h^{(i)}, x_{h+1}^{(i)})}_{i=1}^N$
with $x_h^{(i)}\sim\pi_M,a_h^{(i)}\sim \pi_M, (r_h^{(i)}, x_{h+1}^{(i)})\sim
M^\star_{h}$ of size $N$. Then with probability at least $1-\delta$, for any $h\in [H]$, with probability at least $1-\delta$, we have:
\begin{align*}
\abr{\Ecal_{B}(M, M, h) - {\hEcal}_{B}(M, M, h)} \leq \sqrt{\frac{\log(2H/\delta)}{2N}}.
\end{align*}
\end{lemma}
The result involves a standard application of Hoeffding's inequality with a union bound over $h \in [H]$, which can also be found in \cite{jiang2017contextual}.

\begin{lemma}[Terminate or Explore]
\label{lem:terminate_or_explore}
Suppose that for any round $t$, $\hat{v}^{\pi^t}$ satisfies $\abr{v^{\pi^t} - \hat{v}^{\pi^t}} \leq \epsilon/8$ and $M^{\star}$ is never eliminated. Then in any round $t$, one of the following two statements must hold:
\begin{enumerate}
\item The algorithm does not terminate and there exists a $h\in [H]$ such that $\Ecal_{B}(M^t, M^t, h) \geq \frac{3\epsilon}{8H}$;
\item The algorithm terminates and outputs a policy $\pi^t$ which satisfies $v^{\pi^t} \geq v^{\star} - \epsilon$.
\end{enumerate}
\end{lemma}
\begin{proof}
Let us first consider the situation where the algorithm does not terminate, i.e., $|\hat{v}^{\pi^t} - v_{M^t}| \geq \epsilon/2$. Via~\pref{lem:simulation_lemma}, we must have
\begin{align*}
\sum_{h=1}^H \Ecal_{B}(M^t, M^t, h) \geq \abr{ v^{\pi^t} - v_{M^t} }= \abr{v^{\pi^t} - \hat{v}^{\pi^t}  + \hat{v}^{\pi^t} - v_{M^t}}
 \geq  \abr{\hat{v}^{\pi^t} - v_{M^t} }  - \abr{v^{\pi^t} - \hat{v}^{\pi^t}} \geq 3\epsilon/8.
\end{align*}
By the pigeonhole principle, there must exist $h\in[H]$, such that
\begin{align*}
\Ecal_{B}(M^t, M^t, h) \geq \frac{3\epsilon}{8H},
\end{align*}
so we obtain the first claim.  For the second claim, if the algorithm
terminates at round $t$, we must have $|\hat{v}^{\pi^t} -
v_{M^t}| \leq \epsilon/2$.  Based on the assumption that $M^{\star}$
is never eliminated, and $M^t$ is the optimistic model, we may deduce
\begin{align}
v^{\pi^t} \geq \hat{v}^{\pi^t} - \frac{\epsilon}{8} \geq v_{M^t} - \frac{5\epsilon}{8} \geq v^{\star} - \frac{5\epsilon}{8} \geq v^\star - \epsilon. 
\end{align}
\end{proof}

Recall the definition of the \emph{witness rank} (\pref{def:refined_rank}):
\begin{align*}
\wrank(\kappa,\beta, \Mcal, \Fcal, h) = \inf \crl*{ \rank(A) : \kappa\Ecal_B\left(M, M',h \right) \leq A(M,M') \leq \Wcal(M,M', h), \forall M,M'\in \Mcal}.
\end{align*}
Let us denote $A^{\star}_{\kappa,h}$ as the matrix that achieves the
witness rank $\wrank(\kappa, \beta,\Mcal, \Fcal, h)$ at time step
$h$. Denote the factorization by $A^\star_{\kappa,h}(M,M')
= \inner{\zeta_{h}(M)}{\chi_h(M')}$ with
$\zeta_h,\chi_h \in \RR^{\wrank(\kappa,\beta,\Mcal,\Fcal,h)}$. Finally,
recall that $\beta \geq \max_{M,M',
h} \|\zeta_{h}(M)\|_2\| \chi_{h}(M') \|_2$.

\begin{lemma}
\label{lem:A_terminate_or_explore}
Fix round $t$ and assume that $\abr{{\hEcal}_{B}(M^t, M^t, h) - \Ecal_{B}(M^t, M^t, h)} \leq \frac{\epsilon}{8H}$ for all $h \in [H]$ and $\abr{v^{\pi^t} - \hat{v}^{\pi^t}} \leq \epsilon/8$ hold. If~\pref{alg:main_alg} does not terminate, then we must have $A^{\star}_{\kappa,h_t}(M^t, M^t) \geq \frac{\kappa\epsilon}{8H}$.
\end{lemma}
\begin{proof}
We first verify the existence of $h_t$ in the selection rule~\pref{line:select_h_t} in~\pref{alg:main_alg}.
From~\pref{lem:terminate_or_explore}, we know that there exists $h\in[H]$ such that $\Ecal_{B}(M^t, M^t, h) \geq \frac{3\epsilon}{8H}$, and for this $h$, we have
\begin{align}
{\hEcal}_{B}(M^t, M^t, h) \geq \frac{3\epsilon}{8H} - \frac{\epsilon}{8H} = \frac{\epsilon}{4H}.
\end{align}
While this $h$ may not be the one selected in~\pref{line:select_h_t},
it verifies that $h_t$ exists, and further we do know that for $h_t$
\begin{align*}
\Ecal_{B}(M^t, M^t, h_t) \geq \frac{2\epsilon}{8H} - \frac{\epsilon}{8H} = \frac{\epsilon}{8H},
\end{align*}
Now the constraints defining $A^\star_{\kappa, h_t}$ give $A^\star_{\kappa, h_t}(M^t,M^t) \geq \kappa \Ecal_{B}(M^t, M^t, h_t)$, which proves the lemma.
\end{proof}

Recall the model elimination criteria at round $t$: $\Mcal_t = \{M\in\Mcal_{t-1}: {\tWcal}(M^t, M, h_t)\leq \phi\}$.
\begin{lemma}
\label{lem:iterations_in_A}
Suppose that $ \abr{ {\tWcal}(M^t, M, h_t) - \Wcal(M^t, M, h_t)} \leq \phi$ holds for all $t, h_t$, and $M \in \Mcal$. Then
\begin{enumerate}
\item $M^\star \in \Mcal_{t}$, for all $t$.
\item  Denote $\tMcal_t = \{M\in \tMcal_{t-1}: A^{\star}_{\kappa, h_t}(M^t, M) \leq 2\phi\}$ with $\tMcal_0 = \Mcal$. We have $\Mcal_t \subseteq \tMcal_t$ for all $t$.
\end{enumerate}
\end{lemma}
Observe ${\tMcal}_t$
is defined via the matrix $A^\star_{\kappa,h}$.
\begin{proof}
Recall that we have $\Wcal(M^t, M^\star, h_t) = 0$. Assuming
$M^\star\in\Mcal_{t-1}$ and via the assumption in the statement, for
every $t$, we have
\begin{align*}
{\tWcal}(M^t, M^\star, h_t) \leq {\Wcal}(M^t, M^\star, h_t) + \phi = \phi,
\end{align*}
so $M^\star$ will not be eliminated at round $t$.

For the second result, we know that $\tMcal_{0} = \Mcal$. Assume
inductively that, we have $\Mcal_{t-1}\subset \tMcal_{t-1}$, and let
us prove that $\Mcal_t \subset \tMcal_t$. Towards a contradiction,
let us assume that there exists $M \in \Mcal_t$ such that
$M \notin \tMcal_t$. Since
$M \in \Mcal_t \subset \Mcal_{t-1} \subset \tMcal_{t-1}$, the update
rule for $\tMcal_t$ implies that
\begin{align*}
A_{\kappa,h_t}^\star(M^t,M) > 2\phi.
\end{align*}
But, using the deviation bound and the definition of $A_{\kappa^\star,h}$, we get
\begin{align*}
{\tWcal}(M^t,M,h_t) \geq \Wcal(M^t,M,h_t) - \phi \geq A_{\kappa,h_t}^\star(M^t,M) - \phi > \phi,
\end{align*}
which contradicts the fact that $M \in \Mcal_t$. Thus, by induction we obtain the result.
\end{proof}

With our choice of $\phi
= \frac{\kappa\epsilon}{48H\sqrt{\wrank_\kappa}}$, we may now quantify
the number of rounds of~\pref{alg:main_alg} using $\tMcal_t$.
\begin{lemma} [Iteration complexity]
 \label{lem:number_of_rounds_to_terminate}
Suppose that
 \begin{align*}
 &  \abr{{\tWcal}(M^t, M, h_t) - \Wcal(M^t, M, h_t) } \leq \phi, \qquad \abr{  {\hEcal}_{B}(M^t, M^t, h) - \Ecal_{B}(M^t, M^t, h)   } \leq \frac{\epsilon}{8H},
 \end{align*}
 hold for all $t$, $h_t$, $h \in [H]$, and $M\in \Mcal$, then the number of rounds of~\pref{alg:main_alg} is at most $H\wrank_{\kappa}\log(\frac{\beta}{2\phi})/\log(5/3)$.
\end{lemma}
\begin{proof}
From ~\pref{lem:A_terminate_or_explore}, if the algorithm does not terminate at round $t$, we find $M^t$ and $h_t$ such that
\begin{align*}
A^{\star}_{\kappa,h_t}(M^t, M^t) = \inner{\zeta_{h_t}(M^t)}{\chi_{h_t}(M^t)} \geq \frac{\kappa\epsilon}{8 H} = 6\sqrt{{\wrank_\kappa}} \phi,
\end{align*}
which uses the value of $\phi = \frac{\kappa\epsilon}{48H\sqrt{\wrank_\kappa}}$.

Recall the recursive definition of $\tMcal_{t} = \{M\in \tMcal_{t-1}:
A^{\star}_{\kappa,h_t}(M^t, M) \leq 2\phi\}$
from~\pref{lem:iterations_in_A}. For the analysis, we maintain and
update $H$ origin-centered ellipsoids where the $h^{\textrm{th}}$
ellipsoid contains the set $\{\chi_h(M): M\in \tMcal_t\}$. Denote
$O^h_t$ as the origin-centered minimum volume enclosing ellipsoid
(MVEE) of $\{\chi_{h}(M): M\in \tMcal_t\}$.  At round $t$, for
$\zeta_{h_t}(M^t)$, we just proved that there exists a vector
$\chi_{h_t}(M^t)\in O^{h_t}_{t-1}$ such that
$\inner{\zeta_{h_t}(M^t)}{\chi_{h_t}(M^t)} \geq
6\sqrt{\wrank_{\kappa}}\phi$.  Denote $O^{h_t}_{t-1, +}$ as the origin-centered MVEE
of $\{v\in
O^{h_t}_{t-1}: \inner{\zeta_{h_t}(M^t)}{v} \leq 2\phi \}$. Based
on~\pref{lem:volume_shrink}, and the fact that $O_t^{h_t} \subset
O_{t-1,+}^{h_t}$, by the definition of $\tMcal_t$, we have:
\begin{align*}
\frac{\text{vol}(O_{t}^{h_t})}{\text{vol}(O_{t-1}^{h_t})} \leq \frac{\text{vol}(O_{t-1,+}^{h_t})}{\text{vol}(O_{t-1}^{h_t})} \leq 3/5,
\end{align*}
which shows that if the algorithm does not terminate, then we shrink the volume of $O_{t}^{h_t}$ by a constant factor.

Denote $\Phi \defeq \sup_{M\in \Mcal, h} \|\zeta_h(M)\|_2$ and
$\Psi \defeq \sup_{M\in\Mcal, h}\|\chi_{h}(M)\|_2$. For $O^h_{0}$, we
have $\text{vol}(O^h_0) \leq c_{\wrank_\kappa} \Psi^{\wrank_{\kappa}}$ where $c_{\wrank_{\kappa}}$ is the volume of the unit Euclidean ball in $\wrank_{\kappa}$-dimensions. For any $t$, we
have
\begin{align*}
O^h_{t} &\supseteq \{q \in \mathbb{R}^{\wrank_{\kappa}}: \max_{p: \|p\|_2 \leq \Phi}  \inner{q}{p} \leq 2\phi  \} = \{q\in \mathbb{R}^{\wrank_{\kappa}}:  \|q\|_2 \leq 2\phi/\Phi\}
\end{align*}
Hence, we must have that at termination, $\text{vol}(O^h_{T}) \geq
c_{\wrank_{\kappa}}(2\phi/\Phi)^{\wrank_{\kappa}}$.  Using the volume
of $O^h_{0}$ and the lower bound of the volume of $O^h_{T}$ and the
fact that every round we shrink the volume of $O^{h_t}_t$ by a
constant factor, we must have that for any $h \in [H]$, the number
of rounds for which $h_t = h$ is at most:
\begin{align}
\wrank_{\kappa}\log(\frac{\Phi\Psi}{2\phi})/\log(5/3).
\end{align}
Using the definition $\beta \geq \Phi\Psi$, this gives an iteration complexity of $H\wrank_{\kappa}\log\rbr{\frac{\beta}{2\phi}}/\log(5/3)$.
\end{proof}

We are now ready to prove~\pref{thm:refined_guarantee}. Note that we are using $A^\star_\kappa$, rather than relying on $\Ecal_B$ or $\Wcal$.

\begin{proof} [Proof of ~\pref{thm:refined_guarantee}]

Below we condition on three events: (1) $\abr{ {\tWcal}(M^t, M, h_t)
- \Wcal(M^t, M, h_t) } \leq \phi$ for all $t$ and $M\in\Mcal$, (2)
$\abr{{\hEcal}_{B}(M^t, M^t, h) - \Ecal_{B}(M^t, M^t, h)
} \leq \frac{\epsilon}{8H}$ for all $t$ and $h\in[H]$, and (3) $\abr{
v^{\pi^t} - \hat{v}^{\pi^t} } \leq \epsilon/8$ for all $t$.

Under the first and second condition, from the lemma above, we know
that the algorithm must terminate in at most $T =
{\wrank_{\kappa}}H \log(\beta/(2\phi))/\log(5/3)$ rounds.  Once the
algorithm terminates, based on~\pref{lem:terminate_or_explore}, we
know that we must have found a policy that is
$\epsilon$-optimal.

Now, we show that with our choices for $n, n_e$, and $\phi$, the above
conditions hold with probability at least $1-\delta$.  Based on value
of $n_e = 32\frac{H^2 \log(6HT/\delta)}{\epsilon^2}$,
and~\pref{lem:deviation_hat_v}, we can verify that the third condition
$|v^{\pi^t} - \hat{v}^{\pi^t}| \leq \epsilon/8$ for all $t\in [T]$
with probability $1-\delta/3$, and the condition
$\abr{\hat{\Ecal}_{B}(M^t, M^t, h) - \Ecal_{B}(M^t, M^t, h)
} \leq \epsilon/(8H)$ holds for all $t\in [T]$ and $h\in [H]$ with
probability at least $1-\delta/3$.  Based on the value of $n = 18432
H^2 K \wrank_{\kappa} \log(12T|\Mcal||\Fcal|/\delta)/(\kappa\epsilon)^2$,
the value of $\phi$, and the deviation bound
from~\pref{lem:deviation_bound}, we can verify that the
condition $\left\lvert {\tWcal}(M^t, M, h_t) - \Wcal(M^t,
M, h_t) \right\rvert \leq \phi$ holds for all $t\in [T]$, $M\in\Mcal$
with probability at least $1-\delta/3$.  Together these ensure the
algorithm terminates in $T$ iterations.
The number trajectories is at most $(n_e + n) \cdot T$, and the
result follows by substitute the value of $n_e$, $n$, and $T$.
\end{proof}

\section{Proof of~\pref{thm:tv_sample_complexity}}
 \label{app:tv_results}

\begin{algorithm}
\begin{algorithmic}[1]
\STATE Compute $\tilde{\Fcal}$ from $\Fcal$ and $\Mcal$ via~\pref{eq:tilde_fcal}
\STATE Set $\phi = \kappa\epsilon/(48H\sqrt{\wrank_{\kappa}})$ and $T = H\wrank_{\kappa}\log(\beta/2\phi)/\log(5/3)$
\STATE Set $n_e = \Theta(H^2\log(6HT/\delta)/\epsilon^2)$ and $n = \Theta(H^2K\wrank_{\kappa}\log(12T|\Mcal||\tilde{\Fcal}|/\delta)/(\kappa^2\epsilon^2))$
\STATE Run~\pref{alg:main_alg} with inputs $(\Mcal,\tilde{\Fcal}, n_e, n, \epsilon, \delta, \phi)$ and return the found policy.
\end{algorithmic}
\caption{Extension to $\Fcal$ with Unbounded Complexity. Arguments: $(\Mcal, \Fcal, \epsilon, \delta, \epsilon)$}
\label{alg:unbounded_Fcal}
\end{algorithm}

We are interested in generalizing~\pref{thm:refined_guarantee} to
accommodate a broader class of test functions $\Fcal$, for example
$\{f:\|f\|_{\infty}\leq 1\}$ that induces the total-variation
distance. This class is not a Glivenko-Cantelli class, so it does not
enable uniform convergence, and we cannot simply use empirical mean
estimator as in~\pref{eq:model_err_est_2}.



The key is to define a much smaller function class $\tFcal \subset
\Fcal$ that does enjoy uniform convergence, and at the same time is
expressive enough such that the witnessed model misfit w.r.t.~$\tFcal$ is the same as that w.r.t.~$\Fcal$. To
define $\tFcal$, we need one new definition. For a model $M$ and a
policy $\pi$, we use $x_h \sim (\pi,M)$ to denote that $x_h$ is
sampled by executing $\pi$ \emph{in the model} $M$, instead of the
true environment, for $h$ steps. With this notation, define $f_{\pi, M_1, M_2, h}$ as:
\begin{align*}
\argmax_{f \in \Fcal} \EE\sbr{\EE_{(r,x_{h+1})\sim M_2}\sbr{f(x_h,a_h,r, x_{h+1})} - \EE_{(r,x_{h+1})\sim M_1}\sbr{f(x_h,a_h, r, x_{h+1})} \mid x_h\sim(\pi,M_1), a_h \sim \pi_{M_2}}.
\end{align*}
Note that the maximum over $\Fcal$ is always attained due to the boundedness assumption on $f \in \Fcal$, and hence this definition is without loss of generality. Now we define
\begin{align} \label{eq:tilde_fcal}
\tFcal \defeq \cbr{\pm f_{\pi_{M_3},M_1,M_2,h} : M_1,M_2,M_3 \in \Mcal, h \in [H] }.
\end{align}

This construction is based on the Scheff\'e estimator, which was
originally developed for density estimation in total
variation~\citep{devroye2012combinatorial}. As we have done here, the
idea is to define a smaller function class containing just the
potential maximizers. Importantly, this
smaller function class is computed \emph{independently} of the data,
so there is no risk of overfitting. The main innovation here is that
we extend the Scheff\'{e} estimator to conditional distributions, and
also to handle arbitrary classes $\Fcal$.

\begin{lemma}
\label{lem:scheffe}
For any true model $M^\star \in \Mcal$, policy $\pi_M$,
$h \in [H]$, and target model $M'$, we have
\begin{align*}
\Wcal(M,M',h;\Fcal) = \Wcal(M,M',h;\tFcal).
\end{align*}
Moreover $|\tFcal| \leq 2|\Mcal|^3 H$.
\end{lemma}
\begin{proof}
The bound on $|\tFcal|$ is immediate. For the other claim, by the
realizability assumption for $\Mcal$, $\tFcal$ contains the functions
$f_{\pi_M,M^\star, M',h}$ for each $(M,M',h)$ pair. These are
precisely the test functions that maximize the witness model misfit
for $\Fcal$, and so the IPM induced by $\tFcal$ achieves exactly the
same values.
\end{proof}

Replacing ${\tWcal}(M,M',h)$ in~\pref{eq:model_err_est_2}, which uses
$\Fcal$, to instead use $\tFcal$, we obtain~\pref{alg:unbounded_Fcal}
and~\pref{thm:tv_sample_complexity} as a corollary
to~\pref{thm:refined_guarantee}. The key is that we have eliminated
the dependence on $|\Fcal|$ in the bound.

\section{Lower Bounds and the Separation Result}

\subsection{Proof of~\pref{prop:lower_bound_realizable}}
\label{app:proof_of_lower_bound_realizable}
To prove~\pref{prop:lower_bound_realizable}, we need the following lower bound for best-arm identification in stochastic multi-armed bandits.
\begin{lemma}
[Theorem 2 from \cite{krishnamurthy2016pac}]
For $K\geq 2$, $\epsilon < \sqrt{1/8}$, and any best-arm identification algorithm, there exists a multi-armed bandit problem for which the best arm $i^{\star}$ is $\epsilon$ better than all others, but for which the estimate $\hat{i}$ of the best arm must have $\mathbb{P}[\hat{i}\neq i^{\star}] \geq 1/3$ unless the number of samples collected is at least $K/(72\epsilon^2)$.
\label{lem:lower_bound_MAB}
\end{lemma}


\begin{proof}[Proof of~\pref{prop:lower_bound_realizable}]
Below we explicitly give the construction of $\Mcal$.  Every MDP in
this family shares the same reward function, and actually also shares
the same transition structure for all levels $h \in [H-1]$. The models
only differ in their transition at the last time step.

Fix $H$ and $K\geq 2$. Each MDP $M^{\textbf{a}^\star} \in \Mcal$
corresponds to an action sequence $\textbf{a}^\star = \{a_1^\star,
a_2^\star, \ldots, a_{H-1}^\star\}$ where $a_i^\star \in [K]$. Thus
there are $K^{H-1}$ models. The reward function, which is shared by
all models, is
\begin{align}
R(x) \defeq \one\cbr{x = x^{\star}}
\end{align}
where $x^{\star}$ is a special state that only appears at level $H$.
Let $x'$ denote another special state at level $H$.

For any model $M^{\textbf{a}^{\star}}$, at any level $h < {H-1}$, the
state $x_h$ is simply the history of actions $x_h \triangleq \{a_1,
a_2, \dots a_{h-1}\}$ applied so far, and taking $a\in\Acal$ at state
$x_h$ deterministically transitions to $x_h\circ a \triangleq \{a_1,
a_2, \dots, a_{h-1}, a\}$.
The transition at level $h=H-1$ is defined as follows:
\begin{align}
P^{\textbf{a}^{\star}}( x_{H} | x_{H-1}, a_{H-1}  ) \defeq \begin{cases}
0.5 + \epsilon\one\cbr{x_{H-1}\circ a_{H-1} = \textbf{a}^\star}, & x_{H} = x^{\star} \\
0.5 - \epsilon\one\cbr{x_{H-1}\circ a_{H-1} = \textbf{a}^\star}, & x_{H} = x{'}.
\end{cases}
\end{align}
Thus, in each model $M^{\textbf{a}^\star}$, each action sequence
$\{a_1, a_2, \dots, a_{H-1}\}$ can be regarded as an arm in MAB
problem with $K^{H-1}$ arms, where all the arms yield
$\text{Ber}(0.5)$ reward except for the optimal arm $\textbf{a}^\star$
which yields $\text{Ber}(0.5+\epsilon)$ reward.  In fact, this
construction is information-theoretically equivalent to the
construction used in the standard MAB lower bound, which appears in
the proof of~\pref{lem:lower_bound_MAB}.  That lower bound directly
applies and since we have $K^{H-1}$ arms here, the result follows.
\end{proof}

\subsection{Proof of~\pref{thm:separation}}
\label{app:proof_of_separation}

\pref{thm:separation} has two claims: (1) There exists a family of
MDPs in which \pref{alg:factored_mdp}
achieves polynomial sample complexity, and (2) \emph{Any} model-free algorithm will incur exponential
sample complexity in this family. As we have
discussed, the actual result is stronger in that the model class consists of factored MDPs under a particular structure, and our algorithm can handle any class of factored MDPs with an arbitrary (but known) structure. 

The rest of this subsection is organized as follows: \pref{app:separation_construction} describes the family of MDPs we construct. Since the MDPs obey a factored structure, we can learn this family using our \pref{alg:factored_mdp} and its guarantees in \pref{thm:factored_MDP} immediately applies, which proves the second claim. Then, the first claim is proved in~\pref{app:olive_factored_MDP}, where we leverage the definition of model-free algorithm (\pref{def:model_free}) to induce information-theoretic hardness. 

\begin{figure}[t]
    \centering
    \includegraphics[width=0.6\textwidth]{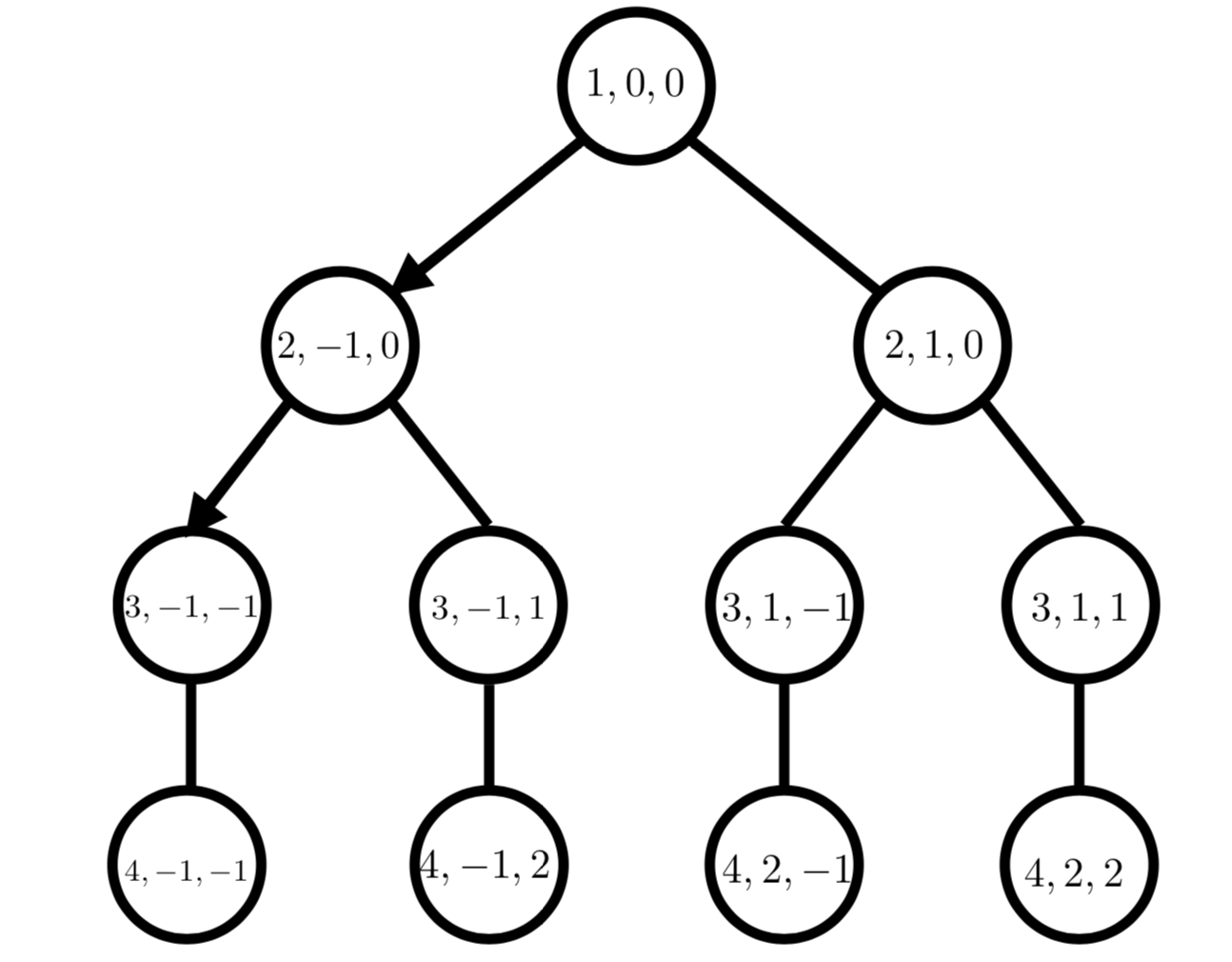}
    \caption{An example of the factored MDP construction in the proof of \pref{thm:separation}, with $d = 2$
      and $H=4$. All models are deterministic, and each model is
      uniquely indexed by a sequence of actions $\bmp$. (Here $\bmp =
      \{-1, -1\}$, as indicated by the black arrows.) The first coordinate in
      each state encodes the level $h$. Each state at level $h \le
      H-1$ encodes the sequence of actions leading to it using bits
      from the second to the last (padded with $0$'s). The last
      transition is designed such that the agent always lands in a
      state that contains ``2'' unless it follows path
      $\bmp$.}
 \label{fig:factored_MDP}
\end{figure}

\subsubsection{Model Class Construction and Sample Efficiency of \pref{alg:factored_mdp}}
\label{app:separation_construction}
\paragraph{Model Class Construction.}
We prove the claim by constructing a family of factored MDPs (recall \pref{eq:factorization}) that share the same reward function $R$ but differ in their transition operators. The set of such transition operators is denoted as $\Pcal$, and we use $P\in\Pcal$ to refer to an MDP instance. 

Fix $d > 2$ and set $H \defeq d+2$.  The state variables take values
in $\Ocal = \{-1,0,1,2\}$. The state space is $\Xcal =
[H]\times\Ocal^d$ with the natural partition across time steps and
the action space is $\Acal = \{-1,+1\}$. The initial state is fixed as
$x=1 \circ [0]^{d}$, where $[a]^d$ stands for a $d$-dimensional vector
where every coordinate is $a$ and $\circ$ denotes concatenation.  Our
model class contains $2^{d}$ models, each of which is uniquely indexed
by an action sequence (or a \emph{path}) of length $d$, $\bmp = \{p_1,
\dots, p_{d}\}$ with $p_i \in \{-1,1\}$. Fixing $\bmp$, we describe
the transition dynamics for $P^{\bmp}$ below. All models share the
same reward function, which will be described afterwards.

In $P^{\bmp}$, the parent of the $i^{\textrm{th}}$ factor is itself so that each
factor evolves independently.  Furthermore, all transitions are
deterministic, so we abuse notation and let $P_h^{\bmp, i}(\cdot,
\cdot)$ denote the deterministic value of the $i^{\textrm{th}}$ factor at time
step $h+1$, as a function of its value at step $h$ and action
$a$. That is, if at time step $h$ we are in state $(h,x_1, \ldots,
x_d)$, upon taking action $a$ we will transition deterministically to
$(h+1,P_h^{\bmp, 1}(x_1, a), \ldots, P_h^{\bmp, d}(x_d, a))$.

Levels $1$ to $H-1$ form a complete binary tree;
see~\pref{fig:factored_MDP} for an illustration. For any layer $h \le
H-2$, 
\begin{align*}
&P^{\bmp, i}_h(v, a) = v, \quad \forall v\in\Ocal, a \in \Acal, \, i\ne h;    \\
&P^{\bmp, i}_h(v, a) = a, \quad \forall v\in\Ocal, a \in \Acal, \, i=h.
\end{align*}
In words, any internal state at level $h \le H-1$ simply encodes the
sequence of actions that leads to it. These transitions \emph{do not}
depend on the planted path $\bmp$ and are identical across all
models. Note that it is not possible to have $x_i=2$ for any $i \in
[d]$, $h \leq H-1$.

Now we define the transition from level $H-1$ to $H$, where each state
only has $1$ action, say $+1$:
\begin{align*}
P^{\bmp, i}_{H-1}(p_i, +1) = p_i, \quad \forall i \in [d], \qquad \textrm{and} \qquad
P^{\bmp, i}_{H-1} (\bar{p_i}, +1) = 2, \quad  i \in [d].
\end{align*}
Here $\bar{p_i}$ is the negation of $p_i$. In words, the state at
level $H$ simply copies the state at level $H-1$, except that the
$i^{\textrm{th}}$ factor will take value $2$ if it disagrees with
$p_i$ (see~\pref{fig:factored_MDP}). Thus, the agent arrives at a
state without the symbol ``2'' at level $H$ only if it follows the action
sequence $\bmp$.

The reward function is shared across all models. Non-zero rewards are
only available at level $H$, where each state only has 1 action.  The
reward is $1$ if $x$ does not contain the symbol ``2" and the reward
is $0$ otherwise. Formally
\begin{align}
R((h,x_1,\ldots,x_d)) \defeq \one\cbr{h = H} \prod_{i=1}^d\one\cbr{x_i \ne 2}. \label{eq:factored_lower_reward}
\end{align}

\paragraph{Sample Efficiency of \pref{alg:factored_mdp}}
For this family of factored MDPs, we have $K=2$ and $d=H-2$. The remaining parameter of interest is $L$, on which we provide a coarse upper bound:  $L \leq d H
|\Acal| |\Ocal|^{2} = O(H^2)$ since $|\parent_i| = 1$ for all $i$ and
$|\Ocal|=4$. Given that our \pref{alg:factored_mdp} works for factored MDPs of any structure, the guarantees in \pref{thm:factored_MDP} immediately applies and we obtain a sample complexity that is polynomial in $H$ and $\log(1/\delta)$. This proves the first claim of \pref{thm:separation}.

\subsubsection{Sample Inefficiency of Model-free Algorithms}
\label{app:factored_MDP_olive}
\label{app:olive_factored_MDP}

We prove the second claim by showing that any model-free
algorithm---that is, any algorithm that always accesses state $x$
exclusively through $[f(x, \cdot)]_{f\in\Gcal}$---will incur
exponential sample complexity when given $\Gcal = \OP(\Pcal)$ as input. 
To show this, we construct another class of non-factored models, such
that (1) learning in this new class is intractable, and (2) the two
families are indistinguishable to any model-free algorithm.  The new
model class is obtained by transforming each $P^{\bmp} \in \Pcal$ into
$\tilde{P}^{\bmp}$. $\tilde{P}^{\bmp}$ has the same state space
and transitions as $P^{\bmp}$, except for the transition from level
$H-1$ to $H$. This last transition is:
\begin{align*}
\tilde{P}^{\bmp}_h((H-1,x_1,\ldots,x_d)) = \left\{
\begin{aligned}
& (H,x_1,\ldots,x_d) & \textrm{ if } x_i = \bmp_i\ \forall i \in [d]\\
& H\circ[2]^d & \textrm{ otherwise.}
\end{aligned}\right.
\end{align*}
The reward function is the same as in the original model class, given
in~\pref{eq:factored_lower_reward}.  This construction is equivalent
to a multi-armed bandit problem with one optimal arm among $2^{H-2}$
arms, so the sample complexity of \emph{any algorithm} (not
necessarily restricted to model-free ones) is
$\Omega(2^H)$.\footnote{Note that the reward function is known and
  non-random, so we do not have any dependence on an accuracy
  parameter $\epsilon$.} In fact this model class is almost identical
to the one used in the proof of~\pref{prop:lower_bound_realizable}.


To prove that the two model families are indistinguishable for
model-free algorithms (\pref{def:model_free}), we show that the
$\Gcal$-profiles in $P^{\bmp}$ are identical to those in
$\tilde{P}^{\bmp}$. This implies that the behavior of a model-free
algorithm is identical in $P^{\bmp}$ and $\tilde{P}^{\bmp}$, so that
the sample complexity must be identical, and hence $\Omega(2^{H})$.


Let $\Mcal = \{P^\bmp\}_{\bmp \in \{-1,1\}^d}$ and $\tMcal = \{\tilde{P}^\bmp\}_{\bmp \in \{-1,1\}^d}$.
Let $\Qcal, \Pi$ to be the $Q$ class and policy classs from $\OP(\Mcal)$,  $\tQcal$ and $\widetilde{\Pi}$ be the policy class from $\OP(\tMcal)$. Since all MDPs of interest have fully deterministic dynamics,
and non-zero rewards only occur at the last step, it suffices to show
that for any deterministic sequence of actions, $\bma$, (1) the final
reward has the same distribution for $P^{\bmp}$ and
$\tilde{P}^{\bmp}$, and (2) the $Q$-profiles $[Q(x_h, \cdot)]_{Q
  \in \Qcal}$ and $[Q(x_h, \cdot)]_{Q \in \tilde\Qcal}$ are equivalent
at all states generated by taking $\bma$ in $P^{\bmp}$ and
$\tilde{P}^{\bmp}$, respectively.\footnote{Since each $\pi\in\Pi$ is just derived from some $Q\in\Qcal$, the equivalence between two $Q$-profiles implies the equivalence between two $\Pi$-profiles, which further implies equivalence in $\Gcal$-profiles.} The reward equivalence is obvious,
so it remains to study the $Q$-profiles.

In $P^{\bmp}$ and at level $H$, since the reward function is shared,
the $Q$-profile is $[1]^{|\Qcal|}$ for the state without ``2" and
$[0]^{|\Qcal|}$ otherwise. Thus, upon taking $\bma=\bmp$ we see the
$Q$-profile $[1]^{|\Qcal|}$ and otherwise we see
$[0]^{|\Qcal|}$. Similarly, in $\tilde{P}^{\bmp}$ the $Q$-profile
is $[0]^{|\tilde{\Qcal}|}$ if the state is $H\circ[2]^d$ and it is
$[1]^{|\tilde{\Qcal}|}$ otherwise. The equivalence here is obvious as
$|\Qcal|=|\tQcal| = 2^d$.

For level $H-1$, no matter the true model path $\bmp$, the $Q^{\bmp'}$
associated with path $\bmp'$ has value $Q^{\bmp'}(\bma, +1) =
\one\cbr{\bma = \bmp'}$ at state $\bma$.  Hence the $Q$-profile at
$\bma$ can be represented as $[\one\cbr{\bma = \bmp'}]_{\bmp' \in
  \{-1, 1\}^d}$, for both $P^\bmp$ and $\tilde{P}^{\bmp}$. Note that
the $Q$-profile does not depend on the true model $\bmp$ because
all models agree on the dynamics before the last step. Similarly, for
$h< H-1$ where each state has two actions $\{-1,1\}$, we have:
\begin{align*}
Q^{\bmp'}(\bma_{1:h-1}, -1) = \one\cbr{\bma_{1:h-1}\circ \text{-1} = \bmp'_{1:h}}, Q^{\bmp'}(\bma_{1:h-1}, 1) = \one\cbr{\bma_{1:h-1}\circ \text{1} = \bmp'_{1:h}}.
\end{align*}
Hence, the
$Q$-profile can be represented as:
\[
[(\one\cbr{\bma_{1:h-1} \circ \text{-1} = \bmp'_{1:h}}, \one\cbr{\bma_{1:h-1} \circ \text{1} = \bmp'_{1:h}})]_{\bmp' \in \{-1, 1\}^d},
\]
again with no difference between $P^\bmp$ and $\tilde{P}^{\bmp}$.
Thus, the model $P^{\bmp}$ and $\tilde{P}^{\bmp}$ induce exactly the
same $Q$-profile for all paths, implying that any model-free
algorithm (in the sense of~\pref{def:model_free}), must behave
identically on both. 
Since the family $\tMcal = \{\tilde{P}^{\bmp}\}_{\bmp}$
admits an information-theoretic sample complexity lower bound of
$\Omega(2^H)$, this same lower bound applies to $\Mcal = \{P^{\bmp}\}_{\bmp}$
for model-free algorithms.

\subsection{Circumventing the Lower Bound via Overparameterization}
\label{app:overparametrization}
In~\pref{thm:separation}, the $\Gcal$-profile resulting from the class
$\Mcal$ (\pref{app:separation_construction}) obfuscates the true
context, a property critical for the separation result.  In
this section, we show that by increasing the expressiveness of the
model class, the induced $\Gcal$-profile could reveal the context and
circument the lower bound. More directly, the lower bound is sensitive
to the choice of $\Gcal$.

This sensitivity raises the question: what is the right choice of
$\Gcal$ for comparing model-based and model-free methods? Since our
construction considers only a small subset of all possible factored
MDPs in correspondence with $\Mcal$, the class $\Gcal = \OP(\Mcal)$ is
the smallest class that guarantees $(Q^\star,\pi^\star) \in
\Gcal$. Therefore, this choice amounts to \emph{proper learning},
which we argue is a natural choice for the purpose of proving lower
bounds. On the other hand, in this section we show that \emph{improper
  learning} or ``overparametrization'' can circumvent the lower bound.

Recall that in~\pref{app:separation_construction}, every model
$M\in\Mcal$ uses the same \emph{true} reward function
from~\pref{eq:factored_lower_reward}.  Here we create a larger
model class $\Mcal'$ and take $\Gcal' = \OP(\Mcal')$. First define a
set of new reward functions:
\begin{align}
R_i^{(-1)}(H, \bmx) \defeq \one\{x_i \neq -1\}, R_i^{(1)}(H, \bmx) \defeq \one\{x_i \neq 1\}, R_i^{(2)}(H, \bmx) \defeq \one\{x_i \neq 2\}~~\forall i \in [d]
\label{eq:new_reward}
\end{align}
We set $\Mcal' = \Mcal \cup \{ (P,R_i^{(j)}): P \in \Pcal, i \in [d],
j \in \{-1,1,2\}\}$.  Namely for every transition structure $P\in
\Pcal$, we pair it with each new reward function.  Note that
$\abr{\Mcal'} = (3d+1)\abr{\Mcal}$.

To circumvent the lower bound, we simply show that with $\Gcal' =
\OP(\Mcal')$, the $\Gcal'$-profile actually reveals the context
$\bmx$. Let us focus on a single coordinate $i \in [d]$, and pick any
transition operator $P \in \Pcal$. Observe that at level $H$ the $Q$
function corresponding to transition operator $P$ is just the
associated reward, and so the $\Gcal$-profile reveals $R_i^{(-1)},
R_i^{(1)}, R_i^{(2)}$. Since we know that $x_i = \{-1,1,2\}$ exactly
one of these will evaluate to zero, allowing us to recover the
$i^{\textrm{th}}$ bit. In particular this allows us to immediately
identify the correct action for time step $h = i$.  Since we can do
this for every $i \in [d]$, using $\Gcal' = \OP(\Mcal')$, we can
easily obtain an algorithm with $O(1)$ sample complexity.

\section{$\Gcal$-profiles in tabular settings}
\label{app:def_value_based}
Here we show that the $\Gcal$-profile yields no information loss in
tabular environments. Thus from the perspective
of~\pref{def:model_free}, model-based and model-free algorithms are
information-theoretically equivalent. 

In tabular settings, the state space $\Xcal$ and action space $\Acal$
are both finite and discrete. It is also standard to use a fully
expressive $Q$-function class, that is $\Qcal = \cbr{Q: \Xcal \times
  \Acal \to [0,1]}$, where the range here arises due to the bounded
reward. We simply set $\Gcal = \Qcal$ here. For each state $x \in \Xcal$ define the function $Q^x$ such
that for all $a \in \Acal$, $Q^x(x',a) = \one\cbr{x = x'}$. Observe
that since $\Qcal$ is fully expressive, we are ensured that  $Q^x \in \Qcal$, $\forall x
\in \Xcal$.

At any state $x' \in \Xcal$, from the $Q$-profile
$\Phi_{\Qcal}(x')$ we can always extract the values $[Q^x(x',a)]_{x
  \in \Xcal}$ for some fixed action $a$. By construction of the $Q^x$
functions, exactly one of these values will be one, while all others
will be zero, and thus we can recover the state $x'$ simply by
examining a few values in $\Phi_{\Qcal}(x')$. In other words, the
mapping $x \mapsto \Phi_{\Qcal}(x)$ is invertible in the tabular case,
and so there is no information lost through the projection. Hence in
tabular setting, one can run classic model-free algorithms such as
${Q}$-learning \citep{watkins1992q} under our definition.

Our definition can also be applied to parameterized $Q$-function class $\Qcal \defeq \{Q(\cdot,\cdot|\theta): \theta\in\Theta \subset \RR^{d}\}$. To perform gradient-based update on the parameter $\theta$, we can use $Q$-profile as follows. Given any state-action pair $(x,a)$, we can approximate $\nabla_{\theta_i} Q(x,a \vert \theta), \forall i \in [d]$, to an arbitrary accuracy, using finite differencing: 
\begin{align*}
\nabla_{\theta_i} Q(x,a \vert \theta) = \lim_{\delta\to 0} \frac{Q(x,a  \vert \theta + \delta e_i) - Q(x,a \vert \theta-\delta e_i)}{2\delta},
\end{align*} where $e_i$ is the vector with zero everywhere except one in the i-th entry, and  $Q(x,a \vert \theta+\delta e_i)$ and $Q(x,a \vert \theta - \delta e_i)$ can be extracted from the $Q$-profile $\Phi_{\Qcal}(x)$. 

The $Q$-profile can also be used to estimate policy gradient on policies induced from the parameterized $Q$ functions. Denote $\Pi_{\Qcal}$ as the policy class induced from $\Qcal$, e.g., $\pi(a|x;\theta) \propto \exp(Q(x,a | \theta))$. Policy gradient method often involves computing the gradient of the log likelihood of the policy (e.g., REINFORCE~\citep{williams1992simple}): $\nabla_{\theta_i} \log (\pi(x|a;\theta)),\forall i\in [d]$, which via chain rule, is determined by $\nabla_{\theta} Q(x,a|\theta)$. Hence, with the finite differencing technique we introduced above for computing $\nabla_{\theta}Q(x,a\vert \theta)$, we can use $Q$-profile to compute $\nabla_{\theta}\log \pi(x|a;\theta)$.

\section{Proof of \pref{thm:factored_MDP}}
\label{app:factored_infinite}

\begin{algorithm}[t]
\begin{algorithmic}[1]
\STATE Run~\pref{alg:main_alg} with $\Fcal$ in \pref{eq:factored_Fcal}, except in~\pref{line:estimate_model_error}, estimate $\tWcal_F(M^t,M',h_t)$ via \pref{eq:What_Wcal_F}.
\end{algorithmic}
\caption{Variant of~\pref{alg:main_alg} for factored MDPs. Arguments: ($\Mcal, n, n_e, \epsilon, \delta, \phi$)}\label{alg:factored_mdp}
\end{algorithm}

Here we prove \pref{thm:factored_MDP}, which states that \pref{alg:factored_mdp} can handle factored MDPs, where $\Mcal$ is the infinite class of all possible factored MDPs under the given structure (i.e., $\{\parent_i\}$ are known). Since the only difference
between two models is their transitions, we use $\Pcal = \{P:
(R^\star, P)\in \Mcal \}$ to represent the model class, and use $P$ and $M$ interchangeably sometimes.

As an input to the algorithm, we supply an $\Fcal$ tailored for factored MDPs that always guarantees Bellman domination (up to a multiplicative constant; see \pref{lem:factored_bellman_domination}). 
In particular,
\begin{align} \label{eq:factored_Fcal}
\Fcal = \{g_1 + \ldots g_d: g_i \in \Gcal_i\},
\end{align}
where each $\Gcal_i = (
\Ocal^{|\parent_i|}\times\Acal \times [H] \times \Ocal \to \{-1, 1\})$. Note that functions in $\Fcal$ operate on $(x,a,r,x')$ and here we are using a slightly incorrect but intuitive notation: $g_i \in \Gcal_i$ takes $(x,a,r,x')$ as input, and only looks at $(x[\parent_i], h, a, x'[i])$ to determine a binary output value, and $\Gcal_i$ is the set of all functions of this form. The IPM induced by $\Fcal$ is the sum of total variation for each factor, and
$$
|\Fcal| = \prod_{i=1}^d 2^{HK|\Ocal|^{1+|\parent_i|} } = 2^L,
$$
so its logarithmic size is polynomial in $L$ and allows uniform convergence. One slightly unusual property of $\Fcal$, compared to how it is used in other results in the main text, is that functions in $\Fcal$ has $\ell_\infty$ norm bounded by $d$ instead of a constant, and this magnitude will be manifested in the sample complexity through concentration bounds.

Besides the specific choice of $\Fcal$, we also need an important change in how we estimate the model misfit $\Wcal_F$ defined in \pref{eq:model_error_factored_mdp}. Since $\Wcal_F$ is defined w.r.t.~uniformly random actions, we change our estimate accordingly by simply dropping the importance weight in \pref{line:estimate_model_error}: Given dataset $\{ (x^{(i)}_h,
a^{(i)}_h,r_h^{(i)}, x^{(i)}_{h+1})\}_{i=1}^n$ generated in \pref{line:gen_data_Wcal} of \pref{alg:main_alg} using roll-in policy $\pi_M$, the new estimator is
\begin{align}
\label{eq:What_Wcal_F}
\widehat{\Wcal}_F(M, M', h) \triangleq \max_{f\in\Fcal} \frac{1}{n}\sum_{i=1}^n \left(  \EE_{(r,x')\sim M'_{h}}[f(x_h^{(i)}, a_h^{(i)}, r, x')] - f(x_h^{(i)},a_h^{(i)}, r_h^{(i)}, x_{h+1}^{(i)})   \right) .
\end{align}

We now state a formal version of \pref{thm:factored_MDP}, which includes the specification of input parameters to \pref{alg:factored_mdp}, and prove it in the remainder of this section.

\begin{theorem}[Formal version of \pref{thm:factored_MDP}]
\label{thm:factored_mdp_full}
Let $M^\star$ be a factored MDP with known structure \pref{eq:factorization}. For any $\epsilon,\delta \in (0,1]$, set $\beta = O(L/K)$, $\kappa = 1/K$, $\wrank_{\kappa,F} = L_h/|\Ocal|$,\footnote{Here we treat $\wrank_{\kappa,F}$ as an algorithm parameter, and its value is an upper bound on the actual witness rank (\ref{corr:factored_refined_rank}).} $\phi = \frac{\kappa\epsilon}{48
H\sqrt{\wrank_{\kappa,F}}}$, and $T =
H \wrank_{\kappa,F}\log(\beta/2\phi)/\log(5/3)$. Run~\pref{alg:factored_mdp} with inputs $(\Mcal, n_e,
n, \epsilon, \delta,\phi)$, where $\Mcal$ is the infinite class of all possible factored MDPs with the given structure, and
\begin{align*}
n_e = \Theta\left( \frac{H^2 \log(HT/\delta)}{\epsilon^2} \right), n
= \Theta\rbr{\frac{d^2 (L\log(\tfrac{dKL}{\epsilon}) + \log(3T/\delta))}{|\Ocal| \epsilon^2}LHK^2},
\end{align*}
then with probability at least $1-\delta$, the algorithm outputs a policy $\pi$ such that $v^{\pi}\geq v^{\star} - \epsilon$, using at most the following number of sample trajectories:
\begin{align*}
\tilde{O}\rbr{\frac{d^2 L^3 H K^2\log(1/\delta)}{\epsilon^2}}.
\end{align*}
\end{theorem}

\subsection{Concentration Result}
\subsubsection{Cover construction}
We prove uniform convergence by discretizing the CPTs in facotred MDPs and constructing a cover of $\Pcal$. Let $\alpha \in (0, 1)$ be the discretization resolution, whose precise value will be set later. For convenience we also assume that $2/\alpha$ is an odd integer. Recall that a factored MDP is fully specified by the CPTs:
$$
\{ P^{(i)}[o \,|\, x[\parent_i], a, h]: o\in\Ocal, x[\parent_i] \in \Ocal^{|\parent_i|}, a\in\Acal, h\in[H] \}_{i=1}^d.
$$
Since each of these probabilities takes value in $[0, 1]$, we start with an improper cover of $\Pcal$ by discretizing this range and considering cover centers $\{\alpha/2, 3\alpha/2, 5\alpha/2, \ldots, 1 - \alpha/2\}$ for each $o\in\Ocal, x[\parent_i] \in \Ocal^{|\parent_i|}, a\in\Acal, h\in[H], i\in[d]$. Note that any number in $[0, 1]$ will be $(\alpha/2)$-close to one of these $(1/\alpha)$ values.
Altogether the discretization yields
$$
\prod_{i=1}^d (1/\alpha)^{HK|\Ocal|^{1+|\parent_i|} } = (1/\alpha)^L
$$
(possibly unnormalized) CPTs.
For the purpose of cover construction, the distance between two CPTs $P$ and $P'$ is defined as
\begin{align} \label{eq:cover_distance}
\max_{o\in\Ocal, x[\parent_i] \in \Ocal^{|\parent_i|}, a\in\Acal, h\in[H]} |P^{(i)}[o \,|\, x[\parent_i], a, h] - P'^{(i)}[o \,|\, x[\parent_i], a, h]|.
\end{align}
Under this distance, any MDP in $\Pcal$ will be $(\alpha/2)$-close to one of the discretized CPTs, hence we say the discretization yields a $(\alpha/2)$-cover of $\Pcal$ with size $(1/\alpha)^L$.

Note that the above cover is improper because many cover centers violate the normalization constraints. We convert this improper cover to a proper one by (1) discarding all cover centers whose $\alpha/2$ radius ball contains no valid models, and (2) replacing every remaining invalid cover center with a valid model in its $\alpha/2$ radius ball. This yields an $\alpha$-cover with size $(1/\alpha)^L$ whose cover centers are all valid models. We denote the set of cover centers as $\Pcal_c$.

\subsubsection{Uniform Convergence of $\tWcal_F$}
\label{sec:W_F_convergence}

Recall the definition of $\tWcal_F$ from \pref{eq:What_Wcal_F}.
Our main concentration result is the following lemma.
\begin{lemma}[Concentration of $\tWcal_F$ in factored MDPs] \label{lem:factored_concentration}
Fix $h$ and model $P\in\Pcal$. Sample a
dataset $\Dcal = \cbr{(x_h^{(i)},a_h^{(i)},r_h^{(i)}, x_{h+1}^{(i)})}_{i=1}^n$
with $x_h^{(i)}\sim\pi_M$, $a_h^{(i)}\sim U(\Acal)$, $(r_h^{(i)}, x_{h+1}^{(i)})\sim
M^\star_{h}$ of size $n$. Fix any $\phi$ and $\delta >0$. With probability at
least $1-\delta$, we have for all $P' \in \Pcal$: \\
$
\abr{{\tWcal}_F(M, M', h) - \Wcal_F(M, M', h)} \leq    \phi,
$
as long as
$$
n \ge \frac{8d^2(L \log(\tfrac{8d|\Ocal|}{\phi}) + \log(2/\delta))}{\phi^2}.
$$
\end{lemma}

We first prove a helper lemma, 
which quantifies the error introduced by approximating $\Pcal$ with $\Pcal_c$.
\begin{lemma} \label{lem:cover_perturb}
For any $P' \in \Pcal$, let $P_c'$ be its closest model in $\Pcal_c$. For any $f \in \Fcal$ and any $x,a$,
\begin{align} \label{eq:cover_perturb}
|\EE_{(r,x')\sim M'_{(x,a)}}[f(x,a,r,x')] - \EE_{(r,x')\sim (M_c')_{(x,a)}}[f(x,a,r,x')]| \le  d |\Ocal|\alpha,
\end{align}
where $(r,x')\sim M'_{(x,a)}$ is the shorthand for $r \sim R'(x,a), x' \sim P'_{(x,a)}$.
\end{lemma}
\begin{proof}
Recall the definition of $f\in\Fcal$ tailored for factored MDPs: $f = g_1 + \dots + g_d$, with each $\|g_i\|_\infty \le 1$. By triangle inequality, we have:
\begin{align*}
\textrm{LHS} \leq &~  \sum_{i=1}^{d} \abr{\EE_{r,x'\sim M'_{(x,a)}}[g_i(x,a,r,x')] - \EE_{r,x'\sim (M'_{c})_{(x,a)}}[g_i(x,a,r,x')]} \\
\leq &~ \sum_{i=1}^d \variation{(P')^{(i)}_{x,a} - (P'_c)^{(i)}_{x,a}}, \tag{H\"older} \\
= &~ \sum_{i=1}^d \sum_{o\in\Ocal} \left|(P')^{(i)}[o|x[\parent_i],a,h] - (P_c')^{(i)}[o|x[\parent_i],a,h] \right| \\
\le &~ d|\Ocal| \alpha. \tag{$\Pcal_c$ yields $\alpha$-cover under distance defined in \pref{eq:cover_distance}}
\end{align*}
\end{proof}

Now we are ready to prove the main concentration result for factored MDPs.
\begin{proof}[Proof of \pref{lem:factored_concentration}]
To argue uniform convergence for $\Pcal$, we first apply Hoeffding's inequality and union bound to $\Pcal_c$. For any fixed $f$, $\tWcal_F$ is the average of i.i.d.~random variables with range $[-\|f\|_\infty, \|f\|_\infty]$. For the $\Fcal$ that we use for factored MDPs, $\|f\|_\infty \le d$, so with probability at least $1-\delta$, $\forall P_c' \in \Pcal_c$,
\begin{align*}
\abr{\widehat{\Wcal}_F(M,M_c',h) - \Wcal_F(M,M_c',h)} \leq 2d \sqrt{\frac{\log\rbr{2|\Pcal_c| |\Fcal|/\delta}}{2n}}.
\end{align*}
We then follow a standard argument to decompose the estimation error for any $P' \in \Pcal$ into three terms:
\begin{align*}
&~ \abr{\widehat{\Wcal}_F(M,M',h) - \Wcal_F(M,M',h)} \le \abr{\widehat{\Wcal}_F(M,M_c',h) - \Wcal_F(M,M_c',h)} \\
&~ + \abr{\widehat{\Wcal}_F(M,M',h) - \widehat{\Wcal}_F(M,M_c',h)} +
\abr{\Wcal_F(M,M',h) - \Wcal_F(M,M_c',h)}.
\end{align*}
We have an upper bound on the first term , so it suffices to upper-bound the other two terms. For the second term,
\begin{align*}
&~ \abr{\widehat{\Wcal}_F(M,M',h) - \widehat{\Wcal}_F(M,M_c',h)}  \\
\leq &~\frac{1}{n} \max_{f\in\Fcal}  \abr{\sum_{i=1}^n \EE_{(r,x')\sim M'_{h}}[f(x_h^{(i)}, a_h^{(i)}, r, x')]  - \sum_{i=1}^n \EE_{(r,x')\sim (M_c')_{h}}[f(x_h^{(i)}, a_h^{(i)}, r, x')]}.
\end{align*} where we use the fact that for any functionals $\mu_1,\mu_2$, we have $\abr{\max_{f}\mu_1(f) - \max_{f}\mu_2(f)} \leq \max_{f}\abr{\mu_1(f) - \mu_2(f)}$.
Now using~\pref{lem:cover_perturb}, we can show that:
\begin{align*}
\abr{\widehat{\Wcal}_F(M,M',h) - \widehat{\Wcal}_F(M,M_c',h)}  \leq \frac{1}{n} (n d\abr{\Ocal}\alpha) = d\abr{\Ocal}\alpha.
\end{align*}

$\abr{\Wcal_F(M,M',h) - \Wcal_F(M,M_c',h)}$ has the same upper bound using exactly the same argument.
So finally we conclude that for all $P' \in \Pcal$,
\begin{align*}
\abr{\widehat{\Wcal}_F(M,M',h) - \Wcal_F(M,M',h)} \leq 2d \sqrt{\frac{\log\rbr{2|\Pcal_c| |\Fcal|/\delta}}{2n}} + 2d |\Ocal| \alpha.
\end{align*}
To guarantee that the deviation is no more than $\phi$, we back up the necessary sample size $n$ from the above expression. Let $\alpha = \tfrac{\phi}{4d|\Ocal|}$, so $2d |\Ocal| \alpha \le \phi/2$. We then want
$$
2d\sqrt{\frac{\log\rbr{2(8d|\Ocal|/\phi)^L/\delta}}{2n}} \le \phi / 2.
$$
It is easy to verify that the sample size given in the lemma statement satisfies this inequality.
\end{proof}

\subsection{Low Witness Rank and Bellman Domination}
In this subsection we establish several important properties of $\Wcal_F$ which will be directly useful in proving \pref{thm:factored_MDP}. To start, we provide a form of $\Wcal_F$ that is equivalent to the definition provided in~\pref{eq:model_error_factored_mdp}. 
The proof is elementary and omitted. 
\begin{lemma}[Alternative definition of $\Wcal_F$] \label{lem:factored_Fcal_induce_sum_TV}
\begin{align*}
\Wcal_F(M, M', h) = \mathbb{E}\left[  \sum_{i=1}^d \variation{{P'}^{(i)}(\cdot|x_h[\parent_{i}],a_h) - P^{\star(i)}(\cdot|x_h[\parent_{i}], a_h)  }   \vert x_h\sim \pi_P, a_h\sim U(\Acal) \right]
\end{align*}
\end{lemma}

Using this lemma, we show two important properties of $\Wcal_F$: (1) that the matrix $\Wcal_F$ has rank at most $\sum_{i=1}^d K
\abr{\Ocal}^{\abr{\parent_{(i)}}}$ (\pref{prop:factored_low_rank}), which is less than $L$, the description length of the factored MDP, and (2) that we can upper-bound $\Ecal_{B}$ using $\Wcal_F$ (\pref{lem:factored_bellman_domination}). For the remainder, it will be convenient to use the notation $L_h = \sum_{i=1}^d K|\Ocal|^{1+\abr{\parent_{(i)}}}$ to be the number of parameters needed to specify the conditional probability table at a single level $h$.
\begin{proposition}
\label{prop:factored_low_rank}
There exists
$\zeta_h:\mathcal{P}\to \mathbb{R}^{L_h/|\Ocal|}$ and $\chi_h:\Pcal\to\RR^{L_h/|\Ocal|}$,
such that for any $P, P'\in\Pcal$, and $h\in [H]$, (recall that $M=(R,P)$ and $M'=(R,P')$)
\begin{align*}
\Wcal_{F}(M,M', h) = \inner{\zeta_h(M)}{\chi_h(M')},
\end{align*}
and $\|\zeta_h(M)\|_2 \cdot \|\chi_h(M')\|_2 \leq O(L_h/K)$.
\end{proposition}
\begin{proof}
Given any policy $\pi$, let us denote $\sd^{\pi}_h(x)\in\Delta(\Xcal_h)$
	as the state distribution resulting from $\pi$ at time step $h$. Then we can write $\sd_h^{\pi}(x) = \sd_h^{\pi}(x[u]) \sd_h^{\pi}(x[-u]| x[u])$, where for a subset $u \subset
	[d]$, we write $x[u]$ to denote the corresponding assignment of those
	state variables in $x$, and $- u = [d] \setminus u$ is the set of remaining variables. 
We use $\sd_h^\pi$ to denote the probability
	mass function and we use $\PP_h^\pi$ to denote the distribution.
	
	
	For any $P, P'\in\Pcal$, we can factorize $\Wcal_{F}(M,M', h)$ as follows:
	\begin{align*}
	&\Wcal_{F}(M, M', h) = \mathbb{E}\left[  \sum_{i=1}^d \variation{ {P'}^{,(i)}(\cdot|x_h[\parent_i],a_h) - P^{\star,(i)}(\cdot|x_h[\parent_{i}], a_h)  }   \vert x_h\sim \pi_M, a_h\sim U(\Acal) \right] \\
	& = \frac{1}{K} \sum_{i=1}^d \sum_{x_h, a}  \sd_h^{\pi_M}(x_h) \variation{ P^{\star,(i)}(\cdot | x_h[\parent_i], a) - P'^{,(i)}(\cdot | x_h[\parent_i],a)} \\
	& = \frac{1}{K} \sum_{i=1}^d \sum_{x_h, a} \sd_h^{\pi_M}(x_h[\parent_i]) \sd_h^{\pi_M}(x_h[- \parent_i]|x_h[\parent_i]) \variation{P^{\star,(i)}(\cdot|x_h[\parent_i], a) - P'^{,(i)}(\cdot|x_h[\parent_i], a)} \\
	& = \frac{1}{K}\sum_{i=1}^d \sum_{a} \sum_{u \in \Ocal^{|\parent_i|}} \PP_h^{\pi_M}[x_h[\parent_i]=u]\variation{P^{\star,(i)}(\cdot |u, a) - P'^{,(i)}(\cdot |u, a)}\\
	& = \inner{\zeta_h(M)}{\chi_h(M')}. 
	\end{align*}
	Here $\zeta_h(M)$ is indexed by $(i, a, u) \in [d] \times	\Acal\times \Ocal^{|\parent_i|}$ with value $$\zeta_h(i,a,u; M)
	\defeq \PP_h^{\pi_M}(x_h[\parent_i] = u)/K.$$
	$\chi_h(M')$ is also
	indexed by $i, a, u$, with value
	$$\chi_h(i,a,u;M') \defeq
	\variation{P^{\star,(i)}(\cdot|u, a) - P'^{,(i)}(\cdot|u, a)}.$$
	Note
	that $\zeta_h$'s value only depends on $M$, while $\chi_h$'s values
	only depend on $M'$. 
	Moreover the
	dimensions of $\zeta_h$ and $\chi_h$ are $\sum_{i=1}^{d} K
	|\Ocal|^{|\parent_i|} = L_h/|\Ocal|$, each entry of $\zeta_h$ is
	bounded by $1/K$, and each entry of $\chi_h$ is at most $2$. Hence, we
	must have $\beta = \sup_{\zeta,\chi} \|\zeta\|_2 \cdot \|\chi\|_2 \leq O(L/K)$. Note that we omit $|\Ocal|$ from the denominator as $L_h$ has a higher exponent on $|\Ocal|$ and the quantity is being treated as a constant in the big-oh notation.
\end{proof}

We now proceed to prove Bellman domination (up to a constant), which relies on the following lemma on the tensorization property of total variation:

\begin{lemma}
	\label{lem:l1_tensor}
	Let $P_1,\ldots,P_n$ and $Q_1,\ldots,Q_n$ be distributions where $P_i
	\in \Delta(\Scal_i)$ for finite sets $\Scal_i$. Define the product
	measures $P^{(n)},Q^{(n)}$ as $P^{(n)}(s_1,\ldots,s_n) \defeq
	\prod_{i=1}^nP_i(s_i)$. Then
	\begin{align*}
	\variation{P^{(n)} - Q^{(n)}} \leq \sum_{i=1}^n \variation{P_i-Q_i}.
	\end{align*}
\end{lemma}

\begin{proof}
	Define
	$W_i\in\Delta(\Scal_1\times\dots\times\Scal_n)$ with $W_i(s_{1:n}) =
	\prod_{j=1}^{i} P_j(s_j) \prod_{j=i+1}^n Q_j(s_j)$, with $i \in
	\cbr{0,\ldots, n}$. This gives $W_0 = Q^{(n)}$, and $W_n = P^{(n)}$.
	Now, by telescoping, we have
	\begin{align*}
	&\variation{ P^{(n)} - Q^{(n)}} = \variation{ W_0 - W_n } \leq \sum_{i=0}^{n-1} \variation{W_{i} - W_{i+1}}.
	\end{align*}
	For $\variation{ W_i - W_{i+1} }$, we have
	\begin{align*}
	\variation{ W_i - W_{i+1} } = \variation{  \prod_{j=1}^i P_j \prod_{j=i+1}^n Q_j  - \prod_{j=1}^{i+1} P_j \prod_{j = i+2}^n Q_j  } = \variation{   Q_{i+1} - P_{i+1}  }. 
	\end{align*}
\end{proof}

With this helper lemma, the following lemma shows the Bellman domination.

\begin{lemma} \label{lem:factored_bellman_domination}
$
\frac{1}{K}\Ecal_{B}(Q_M, Q_{M'}, h) \le \Wcal_{F}(M, M', h).
$
\end{lemma}
\begin{proof}
\begin{align*}
&\Ecal_{B}(Q_M, Q_{M'}, h) = \EE \left[ Q'(x_h, a_h) - r_h - Q'(x_{h+1}, a_{h+1}) \,\big\vert\, x_h\sim \pi_{M}, a_{h:h+1}\sim \pi_{M'} \right] \\
& = \EE\sbr{  \EE_{x_{h+1}\sim P'_{x_h,a_h}}[V_{M'}(x_{h+1})] - \EE_{x_{h+1}\sim P^\star_{x_h,a_h}}[V_{M'}(x_{h+1})]  \mid x_h\sim \pi_M, a_h\sim \pi_{M'} }\\
& \leq \EE \sbr{ \sum_{a_h} \pi_{M'}(a_h|x_h) \abr{ \EE_{x_{h+1}\sim P'_{x_h,a_h}}[V_{M'}(x_{h+1})] - \EE_{x_{h+1}\sim P^\star_{x_h,a_h}}[V_{M'}(x_{h+1})]   }       \mid x_h\sim \pi_{M}  }\\
& \leq \EE \sbr{ \sum_{a_h} \pi_{M'}(a_h|x_h)  \variation{ P'_{x_h,a_h} -P^\star_{x_h,a_h}  } \mid x_h\sim \pi_M } \tag{H\"older and boundedness of $V_{M'}$}\\
& \leq K \EE \sbr{ \mathbb{E}_{a_h\sim U(\Acal)}\left[ \variation{ P'_{x_h,a_h} -P^\star_{x_h,a_h}}\right] \mid x_h\sim \pi_M }
\le K \Wcal_F(M, M', h).
\end{align*}
The first step follows as we expand the definition of $Q'(x_h, a_h)$ by Bellman equation in $M'$, and realize that the immediate reward cancels out with $r_h$ in expectation as the reward function is known.
The last step follows from \pref{lem:l1_tensor} and \pref{lem:factored_Fcal_induce_sum_TV}.
\end{proof}

Combining \pref{prop:factored_low_rank} and \pref{lem:factored_bellman_domination} we arrive at the following corollary:

\begin{corollary}
\label{corr:factored_refined_rank}
Recall the definition of witness rank
in~\pref{def:refined_rank}. Set $\kappa = \frac{1}{K}$, we have \\
$\wrank_{\kappa,F} \triangleq \wrank(\tfrac{1}{K}, \beta, \Mcal, \Fcal, h) \leq L_h/|\Ocal|$,\footnote{Note that we abuse the definition of $\wrank$ to use $\Wcal_F$ instead of $\Wcal$ in~\pref{def:refined_rank} here.} for $\beta = O(L/K)$.
\end{corollary}

\subsection{Proof of \pref{thm:factored_MDP}}
The proof is largely the same as that of \pref{thm:refined_guarantee}, and the only difference is that we use $\widehat{\Wcal}_F$ as the estimator and handle infinite $\Mcal$, whose uniform convergence property is provided in \pref{sec:W_F_convergence}.
Following the proof of~\pref{thm:refined_guarantee}, 
within the high probability events,the algorithm must terminate in
$T= \wrank_{\kappa,F}H\log(\beta/(2\phi))/\log(5/3)$
iterations, where $\kappa=1/K$ and
$\wrank_{\kappa,F}\le L_h/|\Ocal|$ (\pref{corr:factored_refined_rank}) .
As in the previous proof we still set $\phi = \frac{\kappa\epsilon}{48H\sqrt{\wrank_{\kappa,F}}}$. Plugging this value into \pref{lem:factored_concentration} and requiring that each of these estimation events succeeds with probability at least $1-\delta/3T$, we have
\begin{align*}
n  = \Theta\rbr{\frac{d^2 (L\log(\tfrac{dKL}{\epsilon}) + \log(3T/\delta))}{|\Ocal| \epsilon^2}LHK^2} = \tilde{O}\rbr{\frac{d^2L^2HK^2\log(1/\delta)}{\epsilon^2}}.
\end{align*}
Here the $|\Ocal|^2$ on the denominator is dropped due to its negligible magnitude compared to $L$. The rest of the proof is unchanged: Since estimating $\Wcal_F$ requires much more samples than other estimation events, the order of the overall sample complexity is determined by the above expression multiplied by $T$, which gives the desired sample complexity.

\section{Extension to Unknown Witness Rank}
\label{app:uknown_rank}

\pref{alg:main_alg} and its analysis
assumes that we know $\kappa$ and $\wrank_{\kappa}$ (in fact any
finite upper bound of $\wrank_{\kappa}$), which could be a strong
assumption in some cases. In this section, we show that we can apply a
standard doubling trick to handle the situation where $\kappa$ and
$\wrank_{\kappa}$ are unknown.

Let us consider the quantity $\wrank_{\kappa}/\kappa$. Let us denote $\kappa^\star = \arg\min_{\kappa\in (0,1]} \wrank_{\kappa}/{\kappa}$. Note that the sample complexity of~\pref{alg:main_alg} is minimized at $\kappa^\star$.
\pref{alg:double_trick_main_alg} applies the doubling trick to guess $\wrank_{\kappa^\star}$ and $\kappa^\star$ jointly with~\pref{alg:main_alg} as a subroutine.
In the algorithm, $N_{i}$ in the outer loop denotes a guess for $\wrank/\kappa$ as a whole, and in the inner loop we  use $\kappa_{i,j}$ to guess $\kappa$, while setting $\wrank_{i,j} = N_{i}\kappa_{i,j}$, which we use to set the parameter $\phi$ and $n$.
The following theorem characterizes the its sample complexity. 

\begin{algorithm}[t]
\begin{algorithmic}[1]
\FOR{epoch $i = 1, 2, ...$}
	\STATE Set $N_i = 2^{i-1}$ and $\delta_i = \delta/(i(i+1))$
	\FOR{$j = 1 ,2, \dots$}
	\STATE Set $\kappa_{i,j} = (1/2)^{j-1}$, $\delta_{i,j} = \delta_i / (j(j+1))$, and $\wrank_{i,j} = N_i \kappa_{i,j}$
	\IF{$\wrank_{i,j} < 1$}
		\STATE Break
	\ENDIF
	\STATE Set $T_{i,j} = H\wrank_{i,j}\log(\beta/(2\phi))/\log(5/3)$ and $\phi_{i,j} = \epsilon \kappa_{i,j}/(48H\sqrt{\wrank_{i,j}})$
	\STATE Set $n_{e_{i,j}} = \Theta\left( \frac{H^2 \log(6HT_{i,j} \delta_{i,j})}{\epsilon^2} \right)$ and $n_{i,j} = \Theta\left(\frac{ H^2 K \wrank_{i,j} \log(12T_{i,j} |\Mcal||\Fcal|\delta_{i,j})}{\kappa_{i,j}^2\epsilon^2}\right)$
	\STATE Run~\pref{alg:main_alg} with $(\Mcal,\Fcal, n_{i,j}, n_{e_{i,j}}, \epsilon, \delta_{i,j}, \phi)$ for $T_{i,j}$ iterations
	\STATE If~\pref{alg:main_alg} returns a policy, then break and return the policy
	\ENDFOR
\ENDFOR
\end{algorithmic}
\caption{Guessing $\wrank_{\kappa^\star}/\kappa^\star$, Arguments: $(\Mcal, \Fcal, \epsilon, \delta)$}
\label{alg:double_trick_main_alg}
\end{algorithm}

\begin{theorem} \label{thm:double_trick_refined_guarantee}
For any $\epsilon,\delta\in(0,1)$, with $\Mcal$ and $\Fcal$ satisfying \pref{ass:realizable} and~\pref{ass:test_function_realizable}, with probability at least $1-\delta$,~\pref{alg:double_trick_main_alg} terminates and outputs a policy $\pi$ with $v^{\pi} \geq v^\star - \epsilon$, using at most
\begin{align*}
\tilde{O}\left(\frac{H^3 K \wrank^2_{\kappa^\star} \log(|\Mcal||\Fcal|/\delta)}{(\kappa^\star \epsilon)^2} \right) \textrm{ trajectories.}
\end{align*}
\end{theorem}
\begin{proof}
Consider the $j^{\textrm{th}}$ iteration in the $i^{\textrm{th}}$
epoch.  Based on the value of $\phi_{i,j}$, $n_{e_{i,j}}$, $n_{i,j}$,
using~\pref{lem:deviation_bound} and~\pref{lem:deviation_hat_v}, with
probability at least $1-\delta_{i,j}$, for any $t \in [1, T_{i,j}]$
during the run of~\pref{alg:main_alg}, we have
\begin{align}
&|v^{\pi^t} - \hat{v}^{\pi^t}| \leq {\epsilon}/8,  \label{eq:condition_1}\\
& \lvert   \hat{\Ecal}_{B}(M^t, M^t, h) - \Ecal_{B}(M^t, M^t, h)  \rvert \leq \epsilon/(8H),\forall h\in [H], \label{eq:condition_2} \\
& \lvert   \tWcal(M^t, M', h_t) - \Wcal(M^t, M', h_t)   \rvert \leq \phi_{i,j}, \forall M'\in\Mcal. \label{eq:condition_3}
\end{align}
The first condition above ensures that if~\pref{alg:main_alg}
terminates in the $j^{\textrm{th}}$ iteration and the $i^{\textrm{th}}$
epoch and outputs $\pi$, then $\pi$ must be near-optimal, based
on~\pref{lem:terminate_or_explore}. The third inequality above
together with the elimination criteria in~\pref{alg:main_alg}
ensures that $M^\star$ is never eliminated.

Denote $i_\star$ as the epoch where
$2\wrank_{\kappa^\star}/\kappa^\star \leq N_{i_\star} \leq
4\wrank_{\kappa^\star}/\kappa^\star$, and $j_\star$ as the iteration
inside the $i_{\star}^{\textrm{th}}$ epoch where
$\kappa^\star/2 \leq \kappa_{i_\star,j_\star} \leq \kappa^\star$. Since
$\wrank_{i_\star,j_\star} = N_{i_\star}\kappa_{i_\star,j_\star}$, we
have:
\begin{align}
\label{eq:guessed_rank}
\wrank_{\kappa^\star} \leq \wrank_{i_\star,j_\star} \leq 4\wrank_{\kappa^\star}.
\end{align}
Below we condition on the event that $M^\star$ is not eliminated
during any epoch before $i_{\star}$, and any iteration before
$j_\star$ in the $i_\star^{\textrm{th}}$ epoch.  We analyze the
$j_\star^{\textrm{th}}$ iteration in the $i_\star^{\textrm{th}}$ epoch
below. Since $N_{i_\star} = 2^{i_\star - 1}$ and $N_{i_\star} \leq
4\wrank_{\kappa^\star}/\kappa^\star$, we must have $i_\star \leq 1
+ \log_{2}(4\wrank_{\kappa^\star}/\kappa^\star)$. Also note that we have with these settings that $\wrank_{i^\star,j^\star}/(\kappa_{i^\star,j^\star})^2 \geq \wrank_{\kappa^\star}/(\kappa^\star)^2$, so that the number of samples we use at round $(i^\star, j^\star)$is at least as large, and the parameter $\phi$ is no larger than what we would have if we knew $\kappa^\star$ and used it in~\pref{alg:main_alg}.

Based on the value of $\phi_{i_\star,j_\star}$, $n_{i_\star,j_\star}$,
$n_{e_{i_\star,j_\star}}$ and $T_{i_\star,j_\star}$, we know with
probability at least $1- \delta_{i_\star,j_\star}$, for any $t\in [1,
T_{i_\star, j_\star}]$ in the execution of~\pref{alg:main_alg},
inequalities~\pref{eq:condition_1},~\pref{eq:condition_2},
and~\pref{eq:condition_3} hold. Conditioned on this event and since $\phi_{i^\star,j^\star}$ is small enough as observed above, similar to
the proof of~\pref{lem:number_of_rounds_to_terminate}, we can show
that~\pref{alg:main_alg} must terminate in at most
$H\wrank_{\kappa^\star}\log(\beta/2\phi)/\log(5/3)$ many rounds in this iteration.

From~\pref{eq:guessed_rank}, we know that $\wrank_{i_\star,j_\star} \geq \wrank_{\kappa^\star}$, which implies that $T_{i_\star,j_\star} \geq H\wrank_{\kappa^\star}\log(\beta/2\phi)/\log(5/3)$.
In other words, in the $j_\star^{\textrm{th}}$ iteration of the $i_\star^{\textrm{th}}$ epoch, we run~\pref{alg:main_alg} long enough to guarantee that it terminates and outputs a policy.
We have already ensured that if it terminates, it must output a policy $\pi$ with $v^{\pi} \geq v^\star - \epsilon$ (this is true for any $(i,j)$ pair).

Now we calculate the sample complexity. In the $i^{\textrm{th}}$
epoch, since we terminate when $\wrank_{i,j} < 1$, the number of
iterations is at most ${\log_2 N_i} < i$. Hence the number of
trajectories collected in this epoch is at most
\begin{align*}
&\sum_{j=1}^{i}(n_{e_{i,j}} + n_{i,j}) T_{i,j} = \sum_{j=1}^i O\left(   H^3 K \wrank_{i, j}^2 \log(T_{i,j}|\Mcal||\Fcal|/\delta_{i,j})   /(\epsilon\kappa_{i,j})^2\right) \\
&= O\left( i H^3 K N_i^2 \log(T_{i,1}|\Mcal||\Fcal|/\delta_{i, i})  /\epsilon^2 \right),
\end{align*}
where we used the fact that $N_i = \wrank_{i,j}/\kappa_{i,j}$,
$T_{i,1 } \geq T_{i,j}$, and $\delta_{i,i} \leq \delta_{i,j}$.  Note
that $\sum_{i=1}^{i_\star - 1} i N_i^2 = \sum_{i=1}^{i_\star - 1} i
(2^{i-1})^2 \leq (i_\star - 1) (2^{i_\star-1})^2 / 3 = O(i_\star
N_{i_\star}^2)$. Hence the sample complexity in the
$i_\star^{\textrm{th}}$ epoch dominates the total sample complexity,
which is
\begin{align*}
\tilde{O}\left( (1+\log_2 (4\wrank_{\kappa^\star}/\kappa^\star)) H^3 K \wrank^2_{\kappa^\star} \log(T_{i_\star, 0}|\Mcal||\Fcal|/\delta_{i_\star,i_\star}) / (\kappa^\star\epsilon)^2  \right),
\end{align*} where we used the fact that $i_\star \leq 1+\log_2 (2\wrank_{\kappa^\star}/\kappa^\star)$, and $N_{i_\star}\leq 2 \wrank_{\kappa^\star}/\kappa^\star$.
Applying a union bound over $(i,j)$, with $i \leq i_\star$, since we
have $\sum_{i=1}^{i_\star}\sum_{j =
1}^{i} \delta_{i,j}\leq \sum_{i=1}^{i_\star} \delta_{i_\star, 1}
=\sum_{i=1}^{i_\star} \delta/(i(i+1)) \leq \delta$, the failure
probability is at most $\delta$, which proves the theorem.
\end{proof}



\section{Details on Exponential Family Model Class}
\label{app:exponential_family}
For any model $M\in\Mcal$, conditioned on $(x,a)\in\Xcal\times\Acal$, we assume  $M_{x,a} \defeq \exp(\langle  \theta_{x,a}, \mathrm{T}(r,x')  \rangle)/Z(\theta_{x,a})$ with $\theta_{x,a} \in\Theta \subset \RR^{m}$. Without loss of generality, we assume $\|\theta_{x,a}\| \leq 1$, and $\Theta = \{\theta: \|\theta\|\leq 1\}$.
We design $\Vcal = \{\Xcal\times\Acal \to \Theta\}$, i.e., $\Vcal$ contains all mappings from $(\Xcal\times\Acal)$ to $\Theta$.  We design $\Fcal = \{(x,a,r,x')\mapsto \langle v(x,a), \mathrm{T}(r,x')\rangle : v\in\Vcal\}$.
Using~\pref{def:witnessed_model_misfit}, we have:
\begin{align*}
&\Wcal(M, M', h) = \sup_{f\in\Fcal} \EE_{x_h\sim \pi_M, a_h\sim \pi_{M'}}\left[ \EE_{(r,x')\sim M'_{h}}[f(x_h,a_h,r,x')] - \EE_{(r,x')\sim M^\star_h}[f(x_h,a_h,r,x')]     \right] \\
& = \EE_{x_h\sim \pi_M, a_h\sim \pi_{M'}} \left[   \sup_{\theta\in\Theta} \left( \EE_{(r,x')\sim M'_h} \left[\langle\theta, \mathrm{T}(r,x') \rangle\right] - \EE_{(r,x')\sim M^\star_h}\left[\langle  \theta, \mathrm{T}(r,x') \rangle\right] \right)     \right]\\
& = \EE_{x_h\sim \pi_M, a_h \sim \pi_{M'}} \sbr{ \nbr{ \EE_{(r,x')\sim M'_h}[\mathrm{T}(r,x')] - \EE_{(r,x')\sim M^\star_h} [\mathrm{T}(r,x')]}_{\star}},
\end{align*} where the second equality uses the fact that $\Vcal$ contains all possible mappings from $\Xcal\times\Acal\to\Theta$.


We assume that for any $\theta \in \Theta$, the hessian of the log partition function $\nabla^2 \log(Z(\theta))$ is positive definite with eigenvalues bounded between $[\gamma, \beta]$ with $0 \le \gamma \leq \beta$.
Below, we show that under the above assumptions, Bellman domination required for~\pref{ass:test_function_realizable} still holds up to a constant.

\begin{claim}[Bellman Domination for Exponential Families]
In the exponential family setting, we have
\begin{align*}
\frac{\gamma}{2\sqrt{2\beta}}\Ecal_B(M, M', h) \leq \Wcal(M, M', h).
\end{align*}
\end{claim}
\begin{proof}
We leverage Theorem 3.2 from~\citet{gao2018robust}, which implies that 
\begin{align*}
\frac{\gamma}{\sqrt{\beta}}\EE_{x_h\sim \pi_M, a_h\sim \pi_{M'}}\sqrt{D_{KL}( M'_{x_h,a_h} || M^\star_{x_h,a_h}} ) \leq \Wcal(M, M', h).
\end{align*}
By Pinsker's inequality, we have:
\begin{align*}
\frac{\gamma}{\sqrt{2\beta}}\EE_{x_h\sim \pi_M, a_h\sim \pi_{M'}}\variation{M'_{x_h,a_h} - M^\star_{x_h,a_h}} \leq \Wcal(M, M', h),
\end{align*}
where the LHS is the witness model misfit defined using total variation directly~\pref{eq:tv_model_error}.

On the other hand, we know that $r+V_M(x)$ for any $M\in\Mcal$ is upper bounded by $2$ via our regularity assumption on the reward. Hence, the TV-based witness model misfit upper bounds Bellman error as follows:
\begin{align*}
2\EE_{x_h\sim \pi_M, a_h\sim \pi_{M'}}\variation{M'_{x_h,a_h} - M^\star_{x_h,a_h}} \geq \Ecal_{B}(M,M',h),
\end{align*} which concludes the proof. 
\end{proof}

Note that the constant $\gamma/(2\sqrt{2\beta})$ can be absorbed into $\kappa$ in the definition of witness rank (\pref{def:refined_rank}).

\end{document}